\documentclass[nonacm]{acmart}

\usepackage[utf8]{inputenc}
\usepackage{natbib}
\usepackage{amsmath,amsfonts,amsthm}
\usepackage{mathtools}
\usepackage{bm}
\usepackage{subcaption}
\usepackage{color}
\usepackage{hyperref}
\usepackage{accents}
\usepackage{booktabs}
\usepackage{cleveref}
\usepackage{multirow}
\usepackage{xcolor}
\usepackage{colortbl}
\usepackage{arydshln}

\hypersetup{
    colorlinks,
    linkcolor={red!50!black},
    citecolor={blue!50!black},
    urlcolor={blue!80!black},
    breaklinks=true
}

\newtheorem{theorem}{Theorem}[section]
\newtheorem{lemma}{Lemma}[section]
\newtheorem{prop}{Proposition}[section]
\newtheorem{cor}{Corollary}[section]
\theoremstyle{definition}
\newtheorem{definition}{Definition}[section]

\DeclareMathOperator*{\argmax}{arg\,max}
\DeclareMathOperator*{\argmin}{arg\,min}

\newcommand{\ex}[0]{\mathbb{E}}
\newcommand{\pr}[0]{\mathbb{P}}

\newcommand{\xvec}[0]{\mathbf{x}}
\newcommand{\wvec}[0]{\mathbf{w}}
\newcommand{\zvec}[0]{\mathbf{z}}
\newcommand{\noise}[0]{\epsilon}
\newcommand{\evec}[0]{\bm{\xi}}

\newcommand{\yvec}[0]{\mathbf{y}}
\newcommand{\ovec}[0]{\mathbf{o}}

\newcommand{\eff}[0]{\mathbf{e}}

\newcommand{\numb}[0]{k}
\newcommand{\ndist}[0]{F}
\newcommand{\tr}[0]{\intercal}

\newcommand{\bias}[0]{\mathbf{b}}
\newcommand{\vf}[0]{\mathbf{VF}}

\newcommand{\br}[0]{\mathbf{BR}}

\begin{document}

\title{Choosing the Right Weights: \\ Balancing Value, Strategy, and Noise in Recommender Systems
}
\author{Smitha Milli}
\author{Emma Pierson}
\author{Nikhil Garg}
\affiliation{
  \institution{Cornell Tech}
  \city{New York}
  \country{USA}
}

\begin{abstract}
Many recommender systems optimize a linear weighting of different user behaviors, such as clicks, likes, and shares. We analyze the optimal choice of weights from the perspectives of both users and content producers who strategically respond to the weights. We consider three aspects of each potential behavior: value-faithfulness (how well a behavior indicates whether the user values the content), strategy-robustness (how hard it is for producers to manipulate the behavior), and noisiness (how much estimation error there is in predicting the behavior). Our theoretical results show that for users, up-weighting more value-faithful and less noisy behaviors leads to higher utility, while for producers, up-weighting more value-faithful and strategy-robust behaviors leads to higher welfare (and the impact of noise is non-monotonic). Finally, we apply our framework to design weights on Facebook, using a large-scale dataset of approximately 70 million URLs shared on Facebook. Strikingly, we find that our user-optimal weight vector (a) delivers higher user value than a vector not accounting for variance; (b) \textit{also} enhances broader societal outcomes, reducing misinformation and raising the quality of the URL domains, outcomes that were not directly targeted in our theoretical framework.
\end{abstract}

\maketitle

\section{Introduction}
Most widely-used recommender systems are based on prediction and optimization of multiple behavioral signals. For example, a video platform may predict whether a user will click on a video, how long they will watch it, and whether or not they will give it a thumbs-up. These predictions need to then be aggregated into a final score that items for a user will be ranked by. Typically, the aggregation is done through a linear combination of the different signals. For example, leaked documents from TikTok~\citep{smith_2021} described the objective for ranking as $\pr(\text{like)} \cdot w_{\text{like}} + \pr(\text{comment}) \cdot w_{\text{comment}} + \ex[\text{playtime}] \cdot w_{\text{playtime}} + \pr(\text{play}) \cdot w_{\text{play}}$. Twitter also recently open-sourced the exact weights on the ten behaviors they use for ranking~\citep{twitter_2023}.

Unfortunately, the chosen weights can often have unintended consequences. For example, when Facebook introduced emoji reactions, they gave all emoji reactions a weight five times that of the standard thumbs-up. Internal evidence found the high weight on the angry reaction led to more misinformation, toxicity, and low-quality content~\citep{merrill2021five}. Moreover, content producers can especially be affected by these weights. Leaked Facebook documents stated that, \emph{``Research conducted in the EU reveals that political parties feel strongly that the change to the algorithm has forced them to skew negative in their communications on Facebook, with the downstream effect of leading them into more extreme policy positions,''}~\citep{hagey2021facebook,morris_2021}.

Even though they can have a major effect on the emergent dynamics of the platform, these weights are rarely the topic of formal research and there exist few guidelines for system designers on how to choose them. In this paper, we study how to optimally choose weights (for users and producers) when behaviors can vary along three dimensions that designers consider in practice: \emph{value-faithfulness}, \emph{strategy-robustness}, and \emph{noisiness}. Firstly, value-faithfulness is how indicative a behavior is of whether the user values the content or not. This concept is referenced, for example, by TikTok, which has stated that behaviors are ``weighted based on their value to a user''~\citep{tiktok_2020}. Secondly, strategy-robustness refers to how hard it is for producers to manipulate the behavior. A prime example of this is YouTube's shift from focusing less on views and more on explicit user behaviors such as likes and dislikes, in an effort to curb the rise of clickbait video titles~\citep{youtube_2019}. Lastly, noisiness, refers to the variance in machine learning predictions of the behavior. Variance is a common consideration in machine learning and depends, among other factors, on training set size. Netflix increasingly relied on implicit behaviors (e.g. views) over explicit behaviors (e.g. ratings) due to their greater prevalence ~\citep{gomez2015netflix}.


To study the optimal weight design problem, we posit a model in which two producers compete for the attention of one user. The recommender system ranks producers based on a linear combination of predictions of $\numb$ behaviors. However, producers can strategically adapt their items to increase the probability of different user behaviors. User utility depends on being shown a high {value} producer, and a producer's utility is the probability they are ranked highly minus their costs of strategic manipulation.  We find that, for the user, upweighting behaviors that are more value-faithful and less noisy leads to higher utility (and strategy-robustness has no impact), while for producers, upweighting behaviors that are more value-faithful and strategy-robust leads to higher welfare, i.e., higher average utility (and the impact of noise is non-monotonic).

Finally, we apply our framework to empirically design weights on Facebook, using a large-scale dataset of approximately 70 million URLs shared on Facebook. Strikingly, we find that our user-optimal weight vector (a) delivers higher user value than a vector not accounting for variance; (b) \textit{also} reduces misinformation and raises the quality of the URL domains, outcomes that were not directly incorporated in our user value definition.

\section{Related Work}
\paragraph{Designing weights in recommender systems.} 
Many recommender systems use a weighted combination of various behavior predictions to rank items~\citep{smith_2021,twitter_2023}. The choice of weights is typically not automated, and rather is chosen by employees based on performance in A/B tests, insights from user surveys, and qualitative judgement~\citep{tiktok_2020,twitter_2023,cameron_wodinsky_degeurin_germain_2022}. The weights can have a large impact on the emergent dynamics of the platform. For instance, when Facebook modified its weights in 2018 as part of its transition to the ``Meaningful Social Interactions Metric," Buzzfeed CEO Jonah Peretti warned that the changes were promoting the virality of divisive content, thereby incentivizing its production~\citep{hagey2021facebook}.

To bypass the manual weight creation process, some prior work tries to rank content by directly optimizing for users' latent \emph{value} for content~\citep{milli2021optimizing,maghakian2022personalized}. Latent value is unobserved and must be inferred from observed behaviors, but importantly, the relationship between value and the behaviors is empirically learned rather than manually specified through weights. While more elaborate, automated approaches are intriguing, the use of simple linear weights is widespread as it provides system designers with a more interpretable design lever that can be used to shape the platform. Moreover, the models used in these automated approaches are typically not capable of accounting for complex effects like strategic behavior, which humans may be more adept at factoring into their selection. 

Despite the importance of the weights, there exists little formal study of them in the recommender systems context or guidelines on how to select them. In this paper, we study how to choose weights when behaviors can vary along three dimensions: \emph{value-faithfulness}, \emph{strategy-robustness}, and \emph{noisiness}. Value-faithfulness indicates the degree to which a behavior reveals a user's genuine preference for an item. Although defining true value can be challenging~\citep{lyngs2018so}, we focus on users' reflective preferences. Recent research has shown that an overreliance on learning from implicit, less value-faithful signals can cause recommender systems to be misaligned with users' stated, reflective preferences~\citep{lu2018between,kleinberg2022challenge,agan2023automating,milli2023twitter}. Our work aims to offer guidelines for weight selection when behaviors exhibit variability in not just value-faithfulness, but also noisiness and strategy-robustness. For instance, in Section \ref{sec:users}, we demonstrate that, for users, there is a trade-off between choosing value-faithful behaviors and behaviors with lower estimation noise.

\paragraph{Strategic classification, ranking, and recommendation.} Our model considers how strategic producer behavior (alongside estimation noise and behavior value-faithfulness) should affect the design of recommender weights. Strategic behavior by content producers, particularly motivated organizations like news outlets and political parties, has been well-documented \citep{morris_2021,christin2020metrics,hagey2021facebook,smith_2023,meyerson2012youtube}. In our theoretical model, as in practice, producers compete against each other (and so, for example, effort by multiple producers may cancel each other out in equilibria). 

In machine learning, strategic adaptation has been primarily studied in the field of strategic \emph{classification}~\citep{bruckner2012static,hardt2016strategic}. \citet{kleinberg2020classifiers} study how to set linear weights on observed features that can be strategically changed; \citet{braverman_et_al:LIPIcs:2020:12025} show that noisier signals can lead to better equilibrium outcomes when there is heterogeneity in producer's cost functions. Relatively less work in machine learning has focused on the problem of strategic \emph{contests} in which participants must compete to receive desired outcomes, with \citet{liu2022strategic} being a notable exception -- they consider a rank competition setting in a single dimension. There is a rich theory of contests in economics~\citep{hillman1989politically,baye1996all,lazear1981rank,che2000difference,tullock1980cient}, and our model is most similar to the classic model of rank-order tournaments by \citet{lazear1981rank}. To this literature, our work contributes an analysis of how strategic behavior interacts with value-faithfulness and noisiness to influence the recommender system weight design problem.

Prior work has also specifically analyzed strategy in recommender systems~\citep{ben2018game,ben2020content,jagadeesan2022supply,hron2022modeling}, studying properties such as genre formation and producer profit at equilibria~\citep{jagadeesan2022supply} or the algorithmic factors that lead to existence of equilibria~\citep{hron2022modeling}. Relatedly, work on designing ratings systems has considered the importance of variance on consumers and users alike, when designing how users input rating behaviors, and how the platform uses ratings \cite{garg2019designing,garg2021designing,ma2022balancing}. However, the prior work does not model the fact that the ranking objective on a recommender system is typically a linear weighting of \emph{multiple} behaviors with some behaviors being easier to game than others, having higher variance, or being heterogeneously value-faithful. In our work, we focus on the design of these weights.

Finally, beyond recommender systems, literature across economics and operations research considers how and whether to use different feature dimensions---for example, in contracts and standardized testing \citep{holmstrom1987aggregation,holmstrom1991multitask,garg2021standardized,liu2021test}---and similarly characterizes how features' information properties may trade-off with other aspects, such as strategic behavior. 
\section{Model} \label{sec:model}
We model a system with one user and two producers. Each producer  $i \in \{-1, +1\}$ creates an item whose true \textit{value}\footnote{Here, we consider an item's true value to the user to be how the user would value the item upon reflection (as opposed to an immediate, automatic preference \cite{kleinberg2022challenge}).} to the user is $v(i)$, where $v(1) > v(-1)$. Users can interact with a producer's item through $\numb$ different \textit{behaviors}. For example, a user may \emph{click}, \emph{like}, and/or \emph{watch} a video. To rank the producers, the recommender system creates \textit{predictions} $\yvec \in \mathbb{R}^\numb$ of whether the user will engage with the item using each of these $k$ behaviors (using historical data from the user and the producers). It then combines the predictions into a final score $\wvec^T \yvec$, using a \textit{weight vector} $\wvec \in \mathbb{R}^\numb_{\geq 0}$. User utility depends on being shown a high {value} producer, and producer utility depends on being ranked first.

The platform's {design challenge} is to choose weights $\wvec \in \mathbb{R}_{\geq 0}^\numb$, to maximize user utility or producer welfare, i.e., the average utility of both producers. We assume $\|\wvec\|_p = 1$ for some $p$-norm $\|\cdot\|_p$, examples of which include the $\ell^1$ or $\ell^2$ norm.

We postulate that the predictions $\yvec(i) \in \mathbb{R}^\numb$ corresponding to each producer $i\in \{-1, +1\}$ are:
\begin{align}
    \yvec(i) = v(i) + \bias(i) + \evec(i) + \eff(i),
\end{align}
where $v$ reflects the item's \textit{true value}; $\bias_j(i)$ corresponds, for each behavior $j$ and item $i$, to the user's \textit{bias} for engaging with that item; $\evec(i)$ is a noise vector, reflecting variance in the predictions due to finite sample sizes~\citep{domingos2000unified}; $\eff_j(i)$ corresponds to producer $i$'s effort in strategically \textit{manipulating} users to engage in behavior $j$.

\subsection{Behavior Characteristics} We now detail each component of the prediction $\yvec(i)$ introduced above, which represent the three primary behavior characteristics we study: value-faithfulness, noisiness, and strategy-robustness.

\textbf{Value-faithfulness.}  Some behaviors are more indicative of whether the user values the item than others. For example, explicitly liking an item is more indicative of value than simply clicking on the item. To model this, each item $i$ has a behavior-specific bias: $\bias(i) \in \mathbb{R}^\numb$. The sum $v(i) + \bias(i) \in \mathbb{R}^\numb$ captures how likely a user is to engage in a behavior on producer $i$'s item (in the absence of any strategic effort from the producer). Then, the \emph{value-faithfulness} of a behavior $j \in [\numb]$ is 
\begin{align}
    \vf_j & = \ex[\yvec_j(1) - \yvec_j(-1)] \\ & = (v(1)  + \bias_j(1))- (v(-1) + \bias_j(-1)) \,.
\end{align}
The higher a behavior's value-faithfulness, the more likely the user is to engage in that behavior on the higher-valued item compared to the lower-valued item. For example, a \textit{like} has higher value-faithfulness than a \textit{click} because a user is more likely to only \textit{like} a high-valued item while they may \textit{click} on both high and low-valued items. Without loss of generality, we assume that $\vf > \mathbf{0}$ (if a behavior does not have positive value-faithfulness, we can always consider the opposite of the behavior instead, e.g. not clicking instead of clicking).

\textbf{Variance.} The predictions for the behaviors are made by a machine learning model trained on a finite dataset. Some behavior predictions may have higher variance over different realizations of the training data, especially when some behaviors have less historical data than others. We model this estimation error as random behavior-specific noise $\evec(i), \evec(-i) \sim \mathcal{N}(0, \Sigma)$ where $\Sigma \in \mathbb{R}^{\numb \times \numb}$ is a diagonal matrix. Thus, the prediction of a behavior $j$ has higher \emph{variance} than behavior $k$ if $\Sigma_{jj} > \Sigma_{kk}$. In analogy to a traditional bias-variance-noise decomposition~\citep{domingos2000unified}, we aim to capture the \emph{variance} over different realizations of the training data.

\textbf{Strategy-robustness.}   Given a weight vector $\wvec$, producers will strategically adapt their items to get a higher score under $\wvec$. Though we consider an item's true value to be fixed, a producer can put effort $\eff \in \mathbb{R}^k_{\geq 0}$ into increasing the probability that the user interacts with their item with each of the $\numb$ behaviors. For example, without increasing the quality of their content, the producer may craft a clickbait title to entice the user into clicking on it. The producer incurs a cost $c(\eff)$ for their effort. The cost is quadratic with some behaviors being higher cost to manipulate than others: $c(\eff) = \frac{1}{2}\eff^\intercal A\eff$ where $A$ is a diagonal matrix and all entries on the diagonal are unique and positive. We say that a behavior $i$ is more \emph{strategy-robust} than behavior $j$ if $A_{ii} > A_{jj}$.

\subsection{Ranking, Utility, and Equilibria}
Producer one is ranked first if $\wvec ^\tr \yvec(1) - \wvec ^\tr \yvec(-1) > 0$. Or equivalently, producer one is ranked first if $\noise(\wvec) < \ex[\wvec^\tr\yvec(1) - \wvec^\tr\yvec(-1)]$, where  $\noise(\wvec) = \wvec^\tr \evec(-1) - \wvec^\tr \evec(1)$ is the difference in noise terms. Letting $\ndist_\noise$ be the distribution of $\noise(\wvec)$, the probability that producer $i$ is ranked first is
\begin{align}
\pr_{\wvec}(R(i)=1 \mid \eff)  = \begin{cases}
    \ndist_\noise(\ex[\wvec^\tr\yvec(1) - \wvec^\tr\yvec(-1)]) & i = 1 \\
    1- \ndist_\noise(\ex[\wvec^\tr\yvec(1) - \wvec^\tr\yvec(-1)]) & i = -1
\end{cases}
\end{align}
where $R(i)$ is a random variable indicating producer $i$'s rank. 

An individual producer's expected utility is the probability they are ranked first minus the incurred cost of manipulation:
\begin{align}
\mathcal{U}_{\text{prod}}^{i}(\eff(i), \eff(-i); \wvec) =  \pr_{\wvec}(R(i)=1 \mid \eff) - c(\eff(i)) \,.
\end{align}
The \textit{producer welfare} is defined as the average utility of the producers,
\begin{align} \label{eq:prod-welfare}
& \mathcal{W}_{\text{prod}}(\eff(1), \eff(-1); \wvec)\\
& = \frac{1}{2}\sum_{i} \mathcal{U}_{\text{prod}}^i(\eff(i), \eff(-i); \wvec) \\
& = \frac{1}{2} - \frac{\sum_i c(\eff(i))}{2}\,.
\end{align}
The user's utility is the probability the higher-valued producer is ranked first:
\begin{align} \label{eq:user-utility}
\mathcal{U}_{\text{user}}(\eff(1), \eff(-1); \wvec) = \pr_{\wvec}(R(1)=1 \mid \eff) \,.
\end{align}
Producer $i$'s best response given fixed features for the other producer is
\begin{align}
 \br^i(\eff(-i); \wvec) = \argmax_{\mathbf{q} \in \mathbb{R}^\numb_{\geq 0}} \, \mathcal{U}_{\text{prod}}^i(\mathbf{q}, \eff(-i); \wvec)\,.
\end{align}
At a (pure Nash) equilibrium, the best responses of both producers are at a fixed point, defined formally below.
\begin{definition}[Equilibrium]
Given a fixed weight vector $\wvec$, an equilibrium consists of a pair of efforts $(\eff^*(1), \eff^*(-1))$ that satisfy $\br^i(\eff^*(-i); \wvec) = \eff^*(i)$ for $i \in \{-1, +1\}$.
\end{definition}
\section{User Utility Without Strategic Adaptation}

\begin{table*}[t]
\centering
\caption{The Impact of Value-Faithfulness, Variance, and Strategy-Robustness}
\begin{tabular}{|l|c|c|}
\hline
\multirow{2}{*}{\textbf{Aspect of behavior $j$}} & \textbf{Optimal user utility $\mathcal{U}_{\text{user}}^*$} & \textbf{Optimal producer welfare $\mathcal{W}_{\text{prod}}^*$} \\
& (\Cref{cor:user-strategy}) & (\Cref{thm:prod-welf}) \\ \hline
Value-faithfulness $\vf_j$ & Increases & Increases \\
\hline
Variance $\Sigma_{jj}$ & Decreases & Non-monotonic \\
\hline
Strategy-robustness $A_{jj}$ & Constant & Increases \\
\hline
\end{tabular}
\caption*{
\small{The effect of value-faithfulness, variance, and strategy-robustness on user utility under the user-optimal weight vector and producer welfare under the producer-optimal weight vector.}
}
\label{tab:behavior-aspects}
\end{table*}

\label{sec:users}
First, we analyze user utility in the absence of strategic adaptation from producers, i.e., when $\eff \triangleq \mathbf{0}$. Even without any strategic adaptation from producers, it is not obvious what the weights a practitioner should choose are. (We do not analyze producer welfare in the non-strategic setting because producer welfare is constant when producers do not manipulate, i.e., $\mathcal{W}_{\text{prod}}(0, 0; \wvec) = \frac12$.) 

One might assume that for the user it would be best to give the most weight to the most value-faithful behaviors. However, the value-faithful behaviors may not be the easiest to predict, and thus, may introduce more noise into the rankings. Indeed, our analysis, formalized in \Cref{thm:user-w}, shows that the optimal weight vector depends on a trade-off between choosing behaviors that are more value-faithful and choosing behaviors with lower estimation noise. Omitted proofs for this section can be found in \Cref{app:user-proofs}.

\begin{theorem} \label{thm:user-w}
Without any strategic adaptation, the weight vector that maximizes user utility is
\begin{align}
\wvec^* = (\Sigma^{-1}\vf)/\|\Sigma^{-1}\vf\|_p \,.
\end{align}
\end{theorem}

Consequently, the optimal weight on a behavior increases as its value-faithfulness increases or its variance decreases, stated formally below.
\begin{cor}
For any behavior, $j \in [\numb]$, the user-optimal weight on the behavior $\wvec^*_j$ monotonically increases in the behavior's value-faithfulness $\vf_j$ and monotonically decreases in the behavior's variance $\Sigma_{jj}$.
\end{cor}

User utility under the optimal weight vector also increases in value-faithfulness and decreases in variance. Under any weight vector, increasing the value-faithfulness or decreasing the variance of a behavior will increase the probability that the higher-valued producer is ranked first, i.e., user utility. Since this is true for \emph{any} weight vector, it is also true under the optimal weight vector.

\begin{theorem} \label{thm:opt-user-util}
    Let $\mathcal{U}_{\text{user}}(\mathbf{0}, \mathbf{0}; \wvec^*(\vf, \Sigma))$ be user utility under the user-optimal weight vector $\wvec^*$. For any behavior, $j \in [\numb]$, optimal user utility $\mathcal{U}_{\text{user}}(\mathbf{0}, \mathbf{0}; \wvec^*(\vf, \Sigma))$ monotonically increases in the behavior's value-faithfulness $\vf_j$ and monotonically decreases in the behavior's variance $\Sigma_{jj}$.
\end{theorem}

In \Cref{fig:vf-var-tradeoff}, using the closed-form expression for the user-optimal weight vector from \Cref{thm:user-w}, we simulate how the user-optimal weight on a behavior is affected by its value-faithfulness and variance. There are two different behaviors and the weight vector $\wvec \in \mathbb{R}^2_{\geq 0}$ on these behaviors is normalized so that $||\wvec||_1 = 1$. The exact parameters used for all simulations can be found in \Cref{app:sim}. As implied by our theoretical results, the user-optimal weight monotonically increases as a behavior's value-faithfulness increases and monotonically decreases as its variance increases.

\begin{figure}[tb]
  \centering   \includegraphics[width=0.35\textwidth]{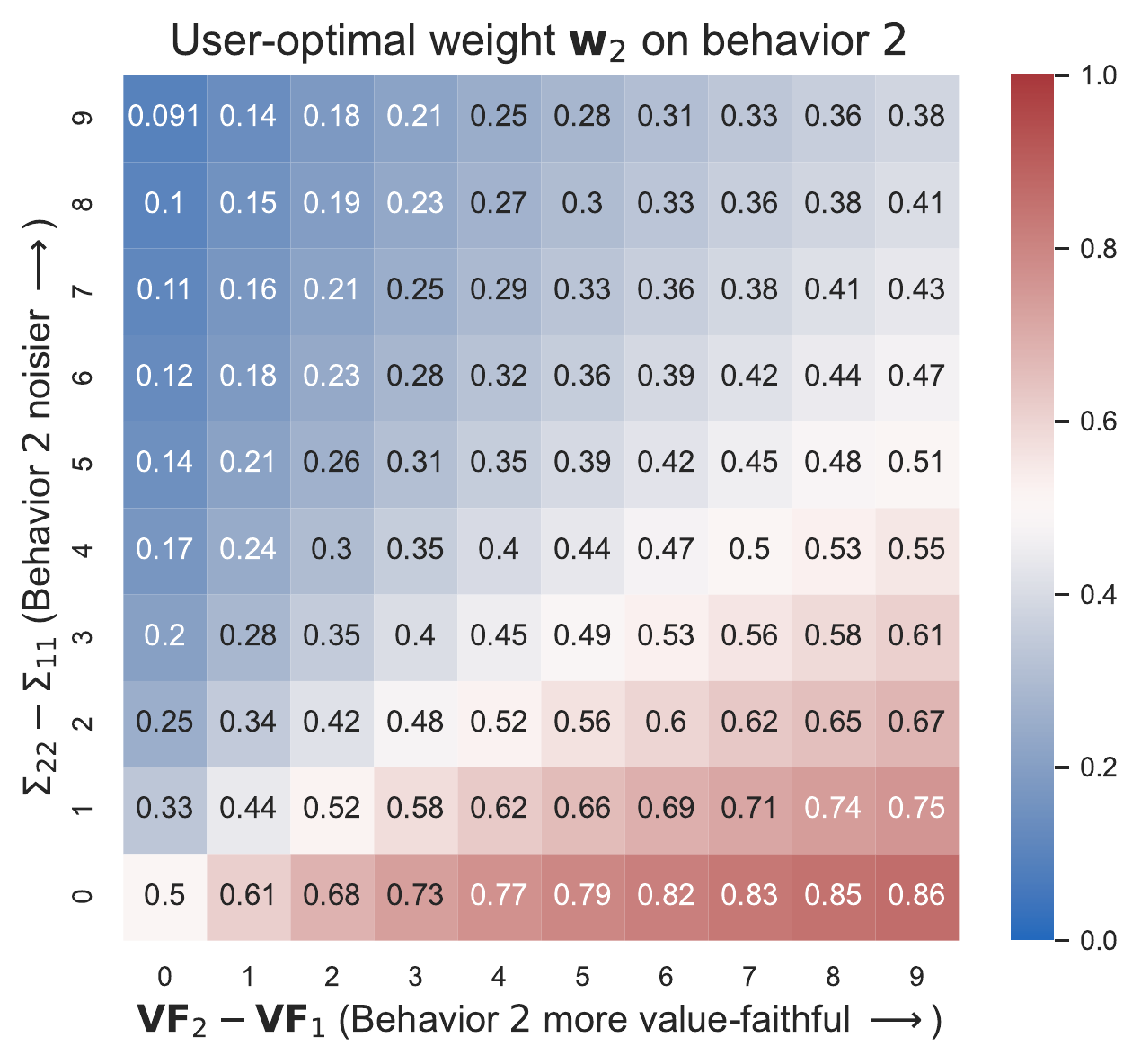}
    \caption{The optimal user weight vector as a function of value-faithfulness and variance. The optimal weight $\wvec_2$ on the second behavior increases as its value-faithfulness increases and decreases as its variance increases.}
    \label{fig:vf-var-tradeoff}
\end{figure}

\begin{figure*}[tb]
    \centering
    \includegraphics[width=0.75\textwidth]{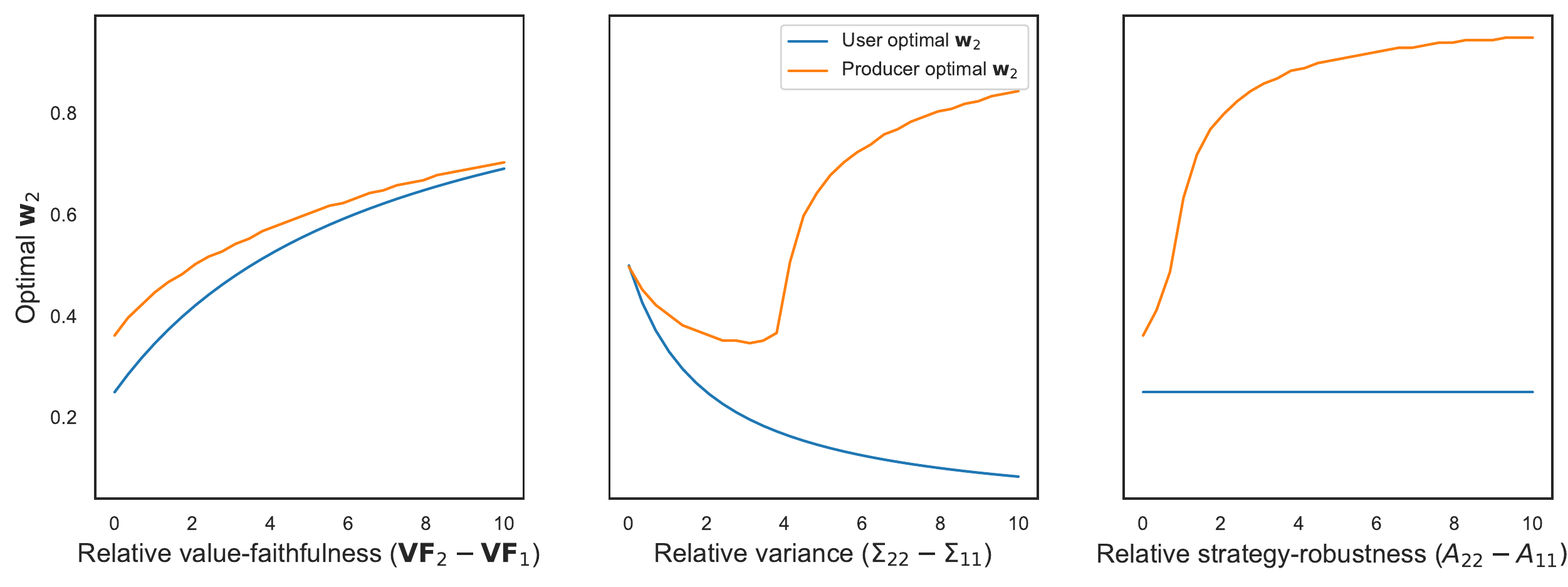}
    \caption{The user-optimal and producer-optimal weight vector as a function of three aspects of behavior: value-faithfulness, variance, and strategy-robustness.}
    \label{fig:main}
\end{figure*}

\section{User and Producer Welfare Under Strategic Adaptation} \label{sec:producers}
In this section, we analyze user utility and producer welfare under the full model described in Section \ref{sec:model}. Omitted proofs can be found in \Cref{app:prod-proofs}. Given the weights $\wvec$ chosen by the designer, producers strategically exert effort $\eff \in \mathbb{R}^\numb_{\geq 0}$ into increasing the probability that the user engages with their item through each behavior. The choice of weights must balance between three aspects of behavior at once: value-faithfulness, strategy-robustness, and noisiness.

In \Cref{prop:eq}, we derive the unique equilibrium strategy for producers and find that the strategy is symmetric, i.e., both producers exert the same effort at equilibrium. 

\begin{prop}[Equilibrium] \label{prop:eq}
The unique equilibrium strategy for both producers is $\eff^*(1) = \eff^*(-1) = f_\epsilon(\wvec^\tr \vf)A^{-1}\wvec$ where $f_\epsilon$ is the density of the difference in noise terms $\epsilon(\wvec) \sim \mathcal{N}(0, 2 \wvec^\tr \Sigma \wvec)$.
\end{prop}

Symmetric equilibria are commonly seen in the literature on \emph{contests} in which agents exert effort towards attaining outcomes that are allocated based on relative rank~\citep{olszewski2016large,bodoh2018college,liu2022strategic}. The intuition for the symmetry in our setting is that producer utility is linear in the probability of being ranked first (which is zero-sum between producers) and the cost of manipulation (where the cost function is the same between producers). Thus, if one producer had found it valuable to expend effort to improve their ranking probability (at the expense of the other), then the other producer would equally have found it valuable to do the same. Hence, they have to have the same effort at equilibrium.

Since the equilibrium is symmetric, the probability that producer one is ranked first is the same as in the non-strategic setting, as producer effort cancels out. Consequently, for users, both the optimal weight vector and user utility under the optimal weight vector remain the same in the strategic setting.  Thus, for users, it is better to up-weigh value-faithful behaviors and down-weigh noisy behaviors (and strategy-robustness has no impact on the optimal weight vector). (This result would not hold in a model in which effort is \textit{productive}, improving producer quality for the user.)

\begin{cor} \label{cor:user-strategy}
Even with strategic adaptation, the user-optimal weight vector is the same as in ~\Cref{thm:user-w}. As in the non-strategic setting (\Cref{thm:opt-user-util}), for any behavior $j$, user utility under the user-optimal weight vector is monotonically increasing in value-faithfulness $\vf_j$, monotonically decreasing in noisiness $\Sigma_{jj}$, and is constant in strategy-robustness $A_{jj}$.
\end{cor}

At equilibrium, the probability that each producer is ranked first is the same as it would be if neither producer exerted any effort. Thus, a producer's effort is essentially wasted effort but is required in order to ``keep up'' with their competitor. Hence, the weight vector that maximizes producer welfare is the one that most disincentives manipulation among producers. Unlike the user-optimal weight vector, it is difficult to find a closed-form solution for the producer-optimal weight vector.  The following proposition shows that the optimal weight vector for producers is the solution to a non-convex optimization problem. Solving for the weight vector requires minimizing the product of a convex quadratic form ($\wvec^\tr A^{-1} \wvec$) and a non-convex term (the Gaussian density $f_\noise(\wvec^\tr \vf)$).
\begin{prop} \label{prop:prod-vector}
    The weight vector that maximizes producer welfare at equilibrium is 
    \begin{align}
        \wvec^* \in \argmin_{\wvec : ||\wvec||=1} f_\noise(\wvec^\tr \vf)^2 \wvec^\tr A^{-1}\wvec
    \end{align}
    where $f_\noise$ is the density of the difference in noise terms \\ $\noise(\wvec) \sim \mathcal{N}(0, 2\wvec^\tr \Sigma \wvec)$.
\end{prop}

We can, however, characterize how producer welfare at equilibrium under the optimal weight vector changes as the three aspects of behavior---value-faithfulness, strategy-robustness, and noisiness---change (\Cref{thm:prod-welf}) We find that producer welfare under the optimal weight vector increases as strategy-robustness and value-faithfulness increase: both strategy-robustness and value-faithfulness disincentivize producer manipulation; strategy-robustness does so directly, and value-faithfulness does so by making the gap between the producers' pre-manipulation scores larger. On the other hand, the relationship between noisiness and producer welfare is non-monotonic. Intuitively, when noise is very high, producer welfare is high because manipulation is disincentivized since the ranking outcome is primarily determined by randomness rather than their scores. Conversely, when noise is very low, producer welfare is also high because producers gain little from any incremental increase in their score, so manipulation is also disincentivized.

\begin{theorem}\label{thm:prod-welf}
Let  $\mathcal{W}_{\text{prod}}^*(\vf, \Sigma, A)$ be the optimal producer welfare given the exogenous parameters for value-faithfulness $\vf$, variance $\Sigma$, and strategy-robustness $A$, i.e., 
    \begin{align}
        & \mathcal{W}_{\text{prod}}^*(\vf, \Sigma, A)\\
        & = \max_{\wvec} \mathcal{W}_{\text{prod}}(\eff_\wvec^*(1), \eff_\wvec^*(-1); \vf, \Sigma, A, \wvec)
    \end{align}
    where $\eff_\wvec^*(1)$ and $\eff_\wvec^*(-1)$ are the unique equilibrium strategies in response to $\wvec$.

 For any behavior, $j \in [k]$, the optimal producer welfare \\ $\mathcal{W}_{\text{prod}}^*(\vf, \Sigma, A)$ is monotonically increasing in the behavior's strategy robustness $A_{jj}$, monotonically increasing in its value-faithfulness $\vf_j$, and is not necessarily monotonic in its variance $\Sigma_{jj}$.

 Furthermore, at the limit, as any of strategy robustness $A_{jj}$, value-faithfulness $\vf_j$, or variance $\Sigma_{jj}$ go to infinity, producer welfare under the optimal weight vector $\mathcal{W}_{\text{prod}}^*(\vf, \Sigma, A)$ reaches the maximum possible value, i.e.,
 \begin{align}
 \lim_{z \rightarrow \infty} \mathcal{W}_{\text{prod}}^*(\vf, \Sigma, A) = 1/2
 \end{align}
 for $z \in \{A_{jj}, \vf_j, \Sigma_{jj}\}$.
\end{theorem}

\begin{table*}[tb!]
\centering
\begin{tabular}{|c|c|c|c|c|c|c|c|c|c|c|}
\hline
\textbf{url\_id} & \textbf{url}             & \textbf{title}   & \textbf{tpfc\_rating}     & \textbf{age}   & \textbf{gender} & \textbf{pol} & \textbf{loves} & \textbf{likes} & \textbf{views} & \textbf{\ldots} \\ \hline
1       & yyy.com/a & Sample headline  & false & 25-34 & F      & -2  & 124   & 34    & 21 & \ldots \\ \hline
2       & zzz.com/b & Sample headline  & null & 35-44 & M      & 1   & -2    & 4 & 46    & \ldots \\ \hline
\vdots  & \vdots   & \vdots       & \vdots & \vdots & \vdots & \vdots & \vdots & \vdots & \vdots & $\ddots$ \\ \hline
\end{tabular}
\caption{Example of the Facebook URLs data~\citep{fburls} used for our experiment in \Cref{sec:fb-exp}. For each URL, we have aggregate engagement counts (`loves', `likes', 'views', etc) by demographic groups defined by age, gender, and political leaning (`pol'). For some URLs, we also have whether the URL was fact-checked as true or false by third-party fact-checkers (`tpfc\_rating'). Note that because of the noise added to the dataset for differential privacy, the number of views on an item may be less than the number of likes or other engagements; it is also possible for any of these counts to be negative.}

\label{tab:fb-urls-dataset}
\end{table*}

In \Cref{fig:main}, we simulate how the user-optimal and producer-optimal weight vector is affected by the three aspects of behavior: value-faithfulness, variance, and strategy-robustness. There are two different behaviors and the weight vector $\wvec \in \mathbb{R}^2_{\geq 0}$ on these behaviors is normalized so that $||\wvec||_1 = 1$. The second behavior scores higher on all three aspects by default, and each subplot isolates the impact of changing one aspect while holding the others constant. For producers, the optimal weight is calculated through numerical approximation.  The exact parameters used for the figure can be found in \Cref{app:sim}.

The producer-optimal weight changes in the same way, as described in \Cref{thm:prod-welf}, that the optimal producer welfare changes as a function of these three aspects. For both users and producers, the optimal weight on a behavior for both users and producers increases as its value-faithfulness increases. On the other hand, as the variance of the behavior increases, the producer-optimal weight changes non-monotonically while the user-optimal weight monotonically decreases. Finally, as the behavior's strategy-robustness increases, the user-optimal weight remains constant while the producer-optimal weight monotonically increases and approaches the maximum possible weight. Overall, though the user and producer-optimal weight vector remain similar as value-faithfulness changes, they diverge drastically as variance or strategy-robustness change. 
\begin{figure*}[t]
    \centering
    \includegraphics[width=1\linewidth]{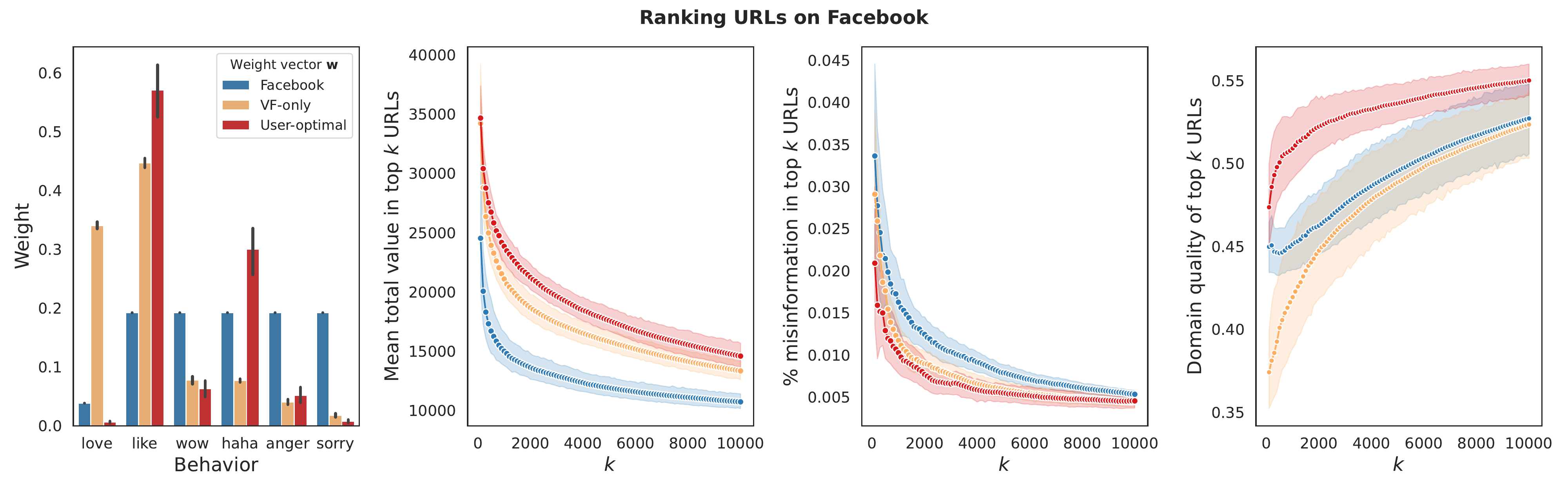}
    \caption[]{The effects of the \texttt{user-optimal}, \texttt{VF-only}, and \texttt{Facebook} weight vector. The \texttt{user-optimal} and \texttt{VF-only} weight vectors were estimated 12 times, each based on data from a different month in 2017\footnotemark. The Facebook weight vector was reported to be unchanged\footnotemark until 2018~\citep{merrill2021five}, and therefore, is constant across the 12 months. The estimated \texttt{user-optimal} and \texttt{VF-only} weight vector from one month of data (e.g. January 2017) was used to rank the next month of URLs (e.g. February 2017). The figure shows the weight vectors (left), the amount of misinformation (center) and domain quality (right) of the top URLs, averaged across months. All error bars represent 95\% bootstrap confidence intervals.}
    \label{fig:fb-exp}
\end{figure*}

\footnotetext{In each month, we filtered to URLs that received at least 100,000 views. On average, this yielded URLs each month, and these URLs accounted for 90\% of the views on URLs in that month.}
\footnotetext{The full Facebook weight vector uses many more behaviors, however, it was reported that the relative weights on these six behaviors did not change until 2018.}

\section{Experiment: Facebook URL recommendation} \label{sec:fb-exp}
We now apply our theoretical insights to the practical challenge of URL recommendation, utilizing a large-scale dataset comprising approximately 70 million URLs shared on Facebook~\citep{fburls}. We compare the performance of the \texttt{user-optimal} weight vector derived from our theoretical analysis, against two baselines: (1) a \texttt{VF-only} weight vector that that focuses solely on value-faithfulness while ignoring variance, and (2) a \texttt{Facebook} weight vector that reflects the actual weighting of behaviors by Facebook. Strikingly, our analysis reveals that the user-optimal weight vector (a) surpasses both baselines in delivering higher user value, even though the \texttt{VF-only} baseline is specifically optimized only for value; (b) \textit{also} enhances broader societal outcomes. Specifically, it contributes to a reduction in misinformation and elevates the quality of the URL domains, outcomes that were not directly targeted in our theoretical framework.

The dataset of URLs that we use is extensive and was released through a collaboration between Facebook and the academic organization Social Science One~\citep{fburls}. It includes all Facebook URLs shared publicly at least 100 times between 2017-2022 along with the aggregate engagement counts by demographic group. The engagement counts have known Gaussian noise added to them for the purposes of differential privacy. Demographic groups are based on age, gender, and political affinity. For example, the dataset might indicate that 143 left-leaning female users aged 18-25 responded to posts with a particular URL using the sad reaction. \Cref{tab:fb-urls-dataset} provides a visual representation of the dataset. For further details on the dataset, please refer to \citet{fburls}. 

We prototype a feature that uses engagement data with URLs to recommend URLs to users on Facebook. On Facebook, there are $k = 6$ different, mutually exclusive ways to react to a post: likes, loves, hahas, wows, sorrys, and angers. We evaluate how the chosen weights on these engagements affect user value and downstream quality of URLs. To simplify our analysis, we consider a nonpersonalized scenario where the predicted probability of engaging with an item using a specific reaction type is the same across users. Nonetheless, our methodology can be easily extended to personalized recommendations by replacing generic predictions with user-specific ones. Finally, we primarily focus on evaluating the impact on users (content consumers), reflecting the theoretical results from \Cref{sec:users} without strategic adaptation. We discuss challenges and implications for content producers in our discussion section (\Cref{sec:disc}).

\paragraph{Empirical overview} To operationalize our framework, we define measures of user value, value-faithfulness, and the prediction noise associated with each behavior according to the user value (\Cref{sec:empdefinevaluef}). Then, we calculate various weight vectors, e.g., an optimal one according to our theory and one that only prioritizes value-faithfulness, using a train time step (\Cref{sec:empcalculatevectors}). Finally, we rank URLs in a future time step according to the various weight vectors (\Cref{sec:emprankurls}), and evaluate the rankings in terms of user value, misinformation, and domain quality (\Cref{sec:empresults}). 

\subsection{Defining value and estimating value-faithfulness and variance}
\label{sec:empdefinevaluef}
We must first define \textit{user value} for a URL. We take an illustrative approach and assume that if a user ``love'' reacted to a URL, then they valued it at unit 1 utility. Then, we say that the value provided by a URL with some other reaction is the population correlation between that reaction and the ``love'' reaction. Accordingly, the \textit{value-faithfulness} of a behavior is also the correlation between it and the ``love'' reaction, e.g., the value-faithfulness for the ``like'' reaction is the correlation between the ``like'' and ``love'' column in \Cref{tab:fb-urls-dataset}. Formally, value-faithfulness of each behavior $i \in [k]$ is 
\begin{align} \label{vf-est}
    \widehat{\vf}_i = \text{corr}(Y_{i}, Y_{\texttt{love}})
\end{align}
where $Y_{i}$ is the random variable that represents a uniformly-drawn sample of the $i$-th behavior across rows of the dataset and $Y_{\texttt{love}}$ is the same for the love behavior.

Then, the true total user value $V_u$ for a URL $u$ at time step $t$ is:
\begin{align} \label{eq:total-value}
V^{t}_u = \sum_{i=1}^{k} y^{t}_{ui} \widehat{\vf}^{t}_i \,,
\end{align}
where $y_{ui}$ is the observed number of times that users reacted to URL $u$ using behavior $i$ and $\widehat{\vf}_i$ is the assumed value-faithfulness of behavior $i$ (\Cref{vf-est}), all measured at time $t$. This value measure can be seen as a simplified version of the principled approach of \citet{milli2021optimizing}, who developed a Bayesian model to determine the value indicated by different behaviors by using one behavior, whose relationship to user value is known, as an anchor (in this case, the love reaction). In practice, the platform itself could also measure value through user surveys (see \Cref{app:measuring-aspects} for more discussion). Indeed, Facebook has conducted surveys to understand how different types of engagement relate to user value~\citep{fbarchive2022}. 

Of course, the platform does not have access to $V^{t}_u$ when ranking URLs. Instead, we say that it ranks URLs according to estimated value faithfulness from a previous time step $\widehat{\vf}^{t-1}_i$, potentially factoring in variance in predicted reaction count $y^{t}_{ui}$. 
%
Next, we describe how, for each behavior $i$, we estimate the variance $\widehat{\Sigma}_{ii}$ associated with the predictions of the behavior. First, we must create a predictor of each behavior from the data. However, the task is complicated by the noise that was added to the dataset for differential privacy. This noise can lead to anomalies where, for example, the number of likes on a URL is lower than the number of views on a URL, or even cases where either of these values is negative. To accurately estimate engagement probabilities while taking into account the differential privacy noise, we employ methods from the statistical literature on errors-in-variables models. The details of our approach can be found in \Cref{app:deming}; in short, given a sample of URLs $\mathcal{S}$, we apply errors-in-variables modeling to derive an estimate $\hat{\beta}_{ui}(\mathcal{S})$ of the probability that users react to URL $u$ with behavior $i$.

The variance $\Sigma_{ii}$ in the prediction of behavior $i$ is then the expected variance of the estimator $\hat{\beta}_{ui}$:
\begin{align} \label{eq:fb-var}
    \Sigma_{ii} = \mathbb{E}[\text{Var}(\hat{\beta}_{ui}(\mathcal{S}))]\,,
\end{align}
where the variance is calculated over the random data sample $\mathcal{S}$ and the expectation is taken across URLs $u$, assuming a uniform random sampling of URLs. To compute our variance estimate $\widehat{\Sigma}_{ii}$, we first sample $n$ URLs. Next, we generate $m$ Bootstrap samples\footnote{In our experiment, $n=1000$ and $m=100$.} from these URLs: $\mathcal{S}_1, \dots, \mathcal{S}_m$. For each URL $u$ and behavior $i$, we calculate the sample variance $\widehat{\text{Var}}_m(\hat{\beta}_{ui})$ of the estimator $\hat{\beta}_{ui}$ using the samples $\mathcal{S}_1, \dots, \mathcal{S}_m$. Then, we average across URLs to get, as an estimate of  \cref{eq:fb-var},
\begin{align}
\widehat{\Sigma}_{ii} = \frac{1}{n} \sum_{u=1}^n \widehat{\text{Var}}_m(\hat{\beta}_{ui}) \,.
\end{align}

\subsection{The user-optimal and baseline weights}
\label{sec:empcalculatevectors}
With the value-faithfulness $\vf$ and variance $\Sigma$ of the different behaviors estimated as above, we apply \Cref{thm:user-w} to calculate the user-optimal weight vector, i.e., 
\begin{align} \label{eq:fb-user-opt}
\wvec^{\texttt{user-optimal}} = (\widehat{\Sigma}^{-1}\widehat{\vf})/\|\widehat{\Sigma}^{-1}\widehat{\vf}|_1\,.
\end{align}
We compare this user-optimal weight vector to two baseline vectors. The first baseline (\texttt{VF-only}) ignores the prediction variance; each behavior's weight is equal to its (normalized) value-faithfulness:
\begin{align} \label{eq:fb-vf-only}
\wvec^{\texttt{VF-only}} = \widehat{\vf}/\|\widehat{\vf}\|_1\,.
\end{align}
Our second baseline $\wvec^{\texttt{Facebook}}$ is based on leaked internal Facebook documents~\citep{merrill2021five} which revealed that when Facebook introduced emoji reactions, they gave all reactions five times the weight of a regular thumbs-up like:
\begin{align} \label{eq:fb-w}
    \wvec^{\texttt{Facebook}}_{i} = \begin{cases}
        1/\alpha & i = \texttt{like} \\
        5/\alpha  & i \neq \texttt{like}
    \end{cases}\,,
\end{align}
where $\alpha=26$ is a normalizer that ensures that $||\wvec^{\texttt{Facebook}}||_1=1$.

\Cref{fig:fb-exp} shows the weight vectors. Notably, the \texttt{user-optimal} vector puts the most weight on the `loves' and `wows' reaction while the \texttt{VF-only} vector puts the most weight on the `loves' and `likes' reaction.

\subsection{Ranking and evaluating the top URLs}
\label{sec:emprankurls}
We use the weight vectors to rank URLs on Facebook in 2017, and evaluate their impact on (1) the total user value generated from the recommended URLs, (2) the amount of misinformation in the recommended URLs, and (3) the quality of the domains recommended.

First, we estimate each weight vector on a per-month basis. Utilizing data from  month $t$, we estimate value-faithfulness $\widehat{\vf}^t$, variance $\widehat{\Sigma}^t$, and the engagement probabilities $\hat{\beta}^t_{ui}$ for that month. For each month $t$, we then calculate the \texttt{user-optimal} and the \texttt{VF-only} weight vectors, $\wvec^{\texttt{user-optimal},t}$ and $\wvec^{\texttt{VF-only}, t}$, by substituting the month-specific estimates of value-faithfulness $\widehat{\vf}^t$ and variance $\widehat{\Sigma}^t$ into \cref{eq:fb-user-opt} and \cref{eq:fb-vf-only}, respectively. The Facebook weight vector was reported to be unchanged\footnotemark until 2018~\citep{merrill2021five}, and therefore, $\wvec^{\texttt{Facebook}, t}$ is constant across the 12 months, as specified in \cref{eq:fb-w}.

We then use weight vectors estimated on data from month $t-1$ to rank URLs in month $t$. In other words,  month $t-1$ serves as the training data for the weight vector, while month $t$ serves as the evaluation data.\footnote{In practice, the score $s_u^{a, t}$ for a URL in month $t$ would also use engagement predictions $\hat{\beta}^{t-1}_{ui}$ based on data from month $t-1$, but since predicting engagement is not our focus, we use the predictions $\hat{\beta}^{t}_{ui}$ based on data from month $t$ in \cref{eq:score}.} In particular, the score $s_u^{a, t}$ for URL $u$ at month $t$ under the weight vector type \\ $a \in \{\texttt{user-optimal}, \texttt{VF-only}, \texttt{Facebook} \}$ is\footnote{On each month, we only consider and score URLs that have received at least 100,000 views that month.}
\begin{align} \label{eq:score}
s_{u}^{a, t} = \sum_{i}^{k} \hat{\beta}^{t}_{ui} \wvec_i^{a, t-1} \,.
\end{align}


As intuition for these scores and value measures, under the \texttt{VF-only} weight vector, the score for each URL $u$ is $s_{u}^{\texttt{VF-only}, t} \propto \sum_{i}^{k} \hat{\beta}^{t}_{ui} \widehat{\vf}^{t-1}_i$. The difference between the \texttt{VF-only} score and the equation for total value $V_u^t$ in \cref{eq:total-value} is that the predicted engagement probability $\hat{\beta}^{t}_{ui}$ is replaced by the observed engagement count $y_{ui}$, and that the weights (derived from \texttt{VF} from the previous month) are replaced by $\widehat{\vf}^{t}_i$ for the current month.

If the predictions of each behavior vary in how noisy they are, then directly weighing them by value-faithfulness, as in  $s_{u}^{\texttt{VF-only}, t}$, may not be the optimal way to estimate total value $V_u$. Unlike the \texttt{VF-only} vector, the \texttt{user-optimal} vector takes into account \emph{both} value-faithfulness and the noise in engagement predictions, and consequently, the scores $s_{u}^{\texttt{user-optimal}, t}$ may offer a more robust estimation of total user value.

\subsection{Empirical Results}
\label{sec:empresults}
\paragraph{Accounting for variance increases user value} As shown in \Cref{fig:fb-exp}, the \texttt{user-optimal} weight vector yields recommendations that generate more total user value than either the \texttt{VF-only} or \texttt{Facebook} weight vector. This finding is surprising since the \texttt{VF-only} vector directly optimizes for value-faithfulness -- in a system without noise, it would maximize user value  by definition. The superior empirical performance of the \texttt{user-optimal} vector relative to the \texttt{VF-only} vector corroborates our theoretical findings which suggest that achieving optimal user value requires consideration of both value-faithfulness and variance.

\paragraph{Our optimal weights decrease misinformation and increase domain quality} Finally, we evaluate the impact of these weight vectors on broader downstream effects. In particular, we evaluate the prevalence of misinformation and the quality of domains in the recommended URLs. Intuitively, one may hope that the ``love'' reaction, our proxy for true value, corresponds to less misinformation and higher domain quality. 

Misinformation is measured based on outcomes from Facebook's third-party fact-checking program which are included in the Facebook URLs dataset. Domain quality is assessed using a latent measure derived by ~\citet{lin2023high} that captures various expert-created domain quality measures that assess dimensions such as factualness, unbiasedness, transparency, etc.

We find that the user-optimal weight vector yields recommendations with less misinformation and higher-quality domains, compared to both the \texttt{VF-only} and \texttt{Facebook} weight vector (\Cref{fig:fb-exp}).  Overall, the results demonstrate the potential for our approach to not only increase user value but also benefit other downstream effects of societal importance.

\section{Discussion} \label{sec:disc}
We analyzed how three aspects of behavior --- value-faithfulness, noisiness, and strategy-robustness --- affect the optimal weight vector and welfare for users and producers. In practice, the weight vector that platforms use is chosen by employees, typically by both product and engineering, based on performance in A/B tests as well as qualitative human judgment ~\citep{twitter_2023,tiktok_2020,merrill2021five,hagey2021facebook}. Understanding how different behaviors compare on the three aspects studied---value-faithfulness, noisiness, and strategy-robustness---can help system designers hone in on the most relevant weight vectors to test. For example, in certain natural settings, our theoretical results imply particular constraints on the user-optimal weight vector. Narrowing the search space of weights is particularly important as a full grid search is typically too expensive to run in real-world applications.

\paragraph{Limitations and open questions.} While our theoretical model accounted for producer effects, measuring producers' strategic adaptation remains challenging in practical applications. Theoretical models of strategic response are often simplistic, and further research is needed to effectively connect theory and practice \cite{patro2022fair}. One prevalent method for accounting for strategic responses is to periodically retrain (in this case, periodically change the weights, as we did in our empirical demonstration), but this may not always be optimal ~\citep{perdomo2020performative}. In the context of recommender systems, conducting producer-side A/B tests to gauge strategic effects can be complex due to the need to avoid interference with ongoing user-side A/B tests and to minimize violations of the Stable Unit Treatment Value Assumption (SUTVA) ~\citep{nandy2021b}. An alternative approach could involve monitoring proxies linked to producer welfare. For example, in our offline experiments ranking Facebook URLs in \Cref{sec:fb-exp}, we measured the effects of the chosen weights on domain quality and misinformation. Arguably, if the weights favor articles from low-quality domains or with more misinformation, then they would likely incentivize producers to create low-quality material. This concern was echoed by Buzzfeed's CEO Jonah Peretti, who cautioned Facebook in 2018 that a change to their weights was promoting the creation of low-quality content~\citep{hagey2021facebook}.

\section*{Acknowledgements}
We thank Gabriel Agostini, Sidhika Balachandar, Ben Laufer, Raj Movva, Kenny Peng, and Luke Thorburn for feedback on a draft version of the paper.

\bibliography{refs}
\bibliographystyle{plainnat}

\clearpage
\begin{appendix}
\onecolumn
\section{Proofs for \Cref{sec:users}} \label{app:user-proofs}

\begin{lemma} \label{lem:pnorm-der}
Let $\|\cdot\|_p$ be a $p$-norm. The partial derivative of $\|\wvec\|_p$ is
\begin{align}
\frac{\partial}{\partial \wvec_i} \|\wvec\|_p = \left(\frac{\vert \wvec_i \vert}{\|\wvec\|_p}\right)^{p-1} \operatorname{sgn}(\wvec_i) \,.
\end{align}
\end{lemma}
\begin{proof}
The $p$-norm of a vector $\wvec$ is equal to $\|\wvec\|_p = \left ( \sum_j |\wvec_j|^p \right )^{(1/p)}$ for some $p \geq 1$. Taking the derivative with respect to a component $\wvec_i$ yields,
\begin{align}
    \frac{\partial}{\partial \wvec_i} \|\wvec\|_p = \frac{1}{p} \left(\sum_j \vert \wvec_j \vert^p\right)^{\frac{1}{p}-1} \cdot p \vert \wvec_i \vert^{p-1} \operatorname{sgn}(\wvec_i) =  \left(\frac{\vert \wvec_i \vert}{\|\wvec\|_p}\right)^{p-1} \operatorname{sgn}(\wvec_i) \,.
    \end{align}
\end{proof}
\subsection*{Proof of \Cref{thm:user-w}}
\begin{proof}
  User utility in the non-strategic setting is equal to
   \begin{align}
   \mathcal{U}_{\text{user}}(\mathbf{0}, \mathbf{0}; \wvec) = \pr_{\wvec}(R(1)=1) = \ndist_\noise(\wvec^\tr \vf) = \frac{1}{2} \left [1 + \mathrm{erf} \left (\frac{\wvec^\tr \vf}{2\sqrt{\wvec^\tr \Sigma \wvec}} \right ) \right ] \,.
    \end{align}
 Since the error function $\mathrm{erf}$ is monotonically increasing on $[0, \infty)$, the optimal weight vector is simply one which solves the following optimization problem (for now, let us ignore the constraint that the weight vector satisfy $\|\wvec\|_p = 1$):
\begin{align}  \label{eq:unconstrained}
    \max_{\wvec \geq 0} g(\wvec) \text{ where }  g(\wvec) = (\wvec^\tr \vf)/\sqrt{\wvec^\tr\Sigma \wvec} \,.
\end{align}
Since the objective function is scale-invariant, i.e., $g(\wvec) = g(a \wvec)$ for any scalar $a \geq 0$, one can rewrite the problem as
\begin{align} \label{eq:constrained}
\max_{\wvec \geq 0} \wvec^\tr \vf \text{ such that } \wvec^\tr\Sigma \wvec = 1
\end{align}
because one can always scale any optimal weight vector for the original problem in \Cref{eq:unconstrained} so that $\wvec^\tr\Sigma \wvec = 1$ is satisfied. Let $\tilde{\wvec} = \Sigma^{1/2}\wvec$ and $\zvec = \Sigma^{-1/2}\vf$. Then, the constrained optimization problem in \Cref{eq:constrained} can be rewritten as
\begin{align} \label{eq:rewritten-constrained}
\max_{\tilde{\wvec} \geq 0} \tilde{\wvec}^\tr \zvec \text{ such that } \tilde{\wvec}^\tr\tilde{\wvec} = 1 \,.
\end{align}
The unique optimal solution to \Cref{eq:rewritten-constrained} is $\tilde{\wvec} = \zvec/ \|\zvec\|_2$. Thus, the unique optimal weight vector that solves \Cref{eq:constrained} is $\wvec = (\Sigma^{-1}\vf)/\|\Sigma^{-1}\vf\|_2$. Therefore, the solution set to the original problem in \Cref{eq:unconstrained} consists of all vectors $\{ \alpha \Sigma^{-1}\vf \mid \alpha > 0\}$. Thus, the optimal weight vector that is unit-norm with respect to the $p$-norm $\|\cdot\|_p$ is $\wvec = (\Sigma^{-1}\vf)/\|\Sigma^{-1}\vf\|_p$.
\end{proof}

\subsection*{Proof of \Cref{cor:user-strategy}}
\begin{proof}
    Define the vector $\zvec = \Sigma^{-1}\vf = (\Sigma_{11}^{-1}\vf_1, \dots, \Sigma_{\numb \numb}^{-1} \vf_\numb)$. By \Cref{thm:user-w}, the user-optimal weight vector is $\wvec^* = \zvec/\|\zvec\|_p$. To prove that the optimal weight vector $\wvec^*$ is increasing in the value-faithfulness $\vf_i$ and decreasing in the variance $\Sigma_i$ of a behavior $j \in [\numb]$, it suffices to prove that for $\zvec \geq 0$, the optimal weight vector $\wvec^*$ is increasing in $\zvec_j$. To do so, we can show that the partial derivative is non-negative:
    \begin{align}
    \frac{\partial}{\partial \zvec_j}\frac{\zvec_j}{\|\zvec\|_p} & = \frac{1}{\|\zvec\|_p^2} \left ( \|\zvec\|_p - \zvec_j \frac{\partial}{\partial \zvec_j}\|\zvec\|_p \right ) \\
    & = \frac{1}{\|\zvec\|_p^2} \left ( \|\zvec\|_p - \zvec_j \left(\frac{ \zvec_j}{\|\zvec\|_p}\right)^{p-1} \right ) \label{eq:norm-derivative} \\
    & = \frac{\|\zvec\|_p^p - \zvec_j^p}{\|\zvec\|_p^{p+1}} \geq 0 \,,
    \end{align}
\end{proof}
where \Cref{eq:norm-derivative} uses \Cref{lem:pnorm-der}, the partial derivative of the $p$-norm, and the fact that $\zvec \geq 0$.

\subsection*{Proof of \Cref{thm:opt-user-util}}
\begin{proof}
    We can write user utility $\mathcal{U}_{\text{user}}$ as
 \begin{align}
   \mathcal{U}_{\text{user}}(\mathbf{0}, \mathbf{0}; \wvec) & = \pr_{\wvec}(R(1)=1)\\ 
   & = \ndist_\noise(\wvec^\tr \vf) \\
   & = \frac{1}{2} \left [1 + \mathrm{erf} \left (\frac{\wvec^\tr \vf}{2\sqrt{\wvec^\tr \Sigma \wvec}} \right ) \right ] \,.
    \end{align}
 For \emph{any} weight vector $\wvec \in \mathbb{R}^\numb_{\geq 0}$, user utility $\mathcal{U}_{\text{user}}(\mathbf{0}, \mathbf{0}; \wvec)$ is monotonically increasing in $\vf_j$ and  monotonically decreasing in $\Sigma_{jj}$ for any behavior $j \in [\numb]$. Thus, user utility under the \emph{optimal} weight vector must also be monotonically increasing in $\vf_j$ and monotonically decreasing in $\Sigma_{jj}$.
\end{proof}
\section{Proofs for \Cref{sec:producers}} \label{app:prod-proofs}

\subsection*{Proof of \Cref{prop:eq}}
\begin{proof}
    The utility for producer $i$ is equal to
    \begin{align}
        \mathcal{U}_{\text{prod}}^{i}(\eff(i), \eff(-i); \wvec) = \begin{cases} 
         F_\noise(\ex[\wvec^\tr\yvec(1) - \wvec^\tr\yvec(-1)]) - c(\eff(1)) & i = 1 \\
        1 - F_\noise(\ex[\wvec^\tr\yvec(1) - \wvec^\tr\yvec(-1)]) - c(\eff(-1)) & i = -1
        \end{cases} \,,
    \end{align}
    where $F_\noise$ is the CDF of the difference in noise terms $\noise(\wvec) = \wvec^\tr \evec(-1) - \wvec^\tr \evec(1) \sim \mathcal{N}(0, 2 \wvec^\tr \Sigma \wvec)$ and the mean difference in producer scores is equal to $\ex[\wvec^\tr\yvec(1) - \wvec^\tr\yvec(-1)] = \wvec^\tr\vf + \wvec^\tr\eff(1) - \wvec^\tr\eff(-1)$.

    At the equilibrium, the first-order conditions for both producers must be satisfied:
    \begin{align}
        \label{eq:prod1_foc} & \nabla_{\eff(1)} \mathcal{U}_{\text{prod}}^{1}(\eff(1), \eff(-1); \wvec) = f_\noise(\wvec^\tr \vf + \wvec^\tr\eff(1) - \wvec^\tr\eff(-1)) \wvec - A \eff(1) = 0 \,, \\ 
        & \nabla_{\eff(-1)} \mathcal{U}_{\text{prod}}^{-1}(\eff(-1), \eff(1); \wvec) = f_\noise(\wvec^\tr \vf + \wvec^\tr\eff(1) - \wvec^\tr\eff(-1)) \wvec - A \eff(-1) = 0 \,,\label{eq:prod0_foc}
    \end{align}
    where $f_\noise$ is the density of $\noise(\wvec)$. Subtracting Equations \ref{eq:prod1_foc} and \ref{eq:prod0_foc} shows that the equilibrium strategy is symmetric, i.e, $\eff(1) = \eff(-1)$ at equilibrium. Substituting $\eff(1) = \eff(-1)$ into either equation yields $\eff(i) = f_\wvec(\wvec^\tr \vf)A^{-1}\wvec$ for $i \in \{-1, +1\}$. 
    
    To prove sufficiency, we need to consider the second-order conditions and show that the Hessian of each producer's utility is negative-definite at the equilibrium efforts. The Hessian is given by
    \begin{align}
        \nabla^2_{\eff(i)} \mathcal{U}_{\text{prod}}^{i}(\eff(i), \eff(-i); \wvec) & = \nabla_{\eff(i)} f_\noise(\wvec^\tr \vf + \wvec^\tr\eff(i) - \wvec^\tr\eff(-i)) \wvec - A \eff(i) \\
        & = f'_\noise(\wvec^\tr \vf + \wvec^\tr\eff(i) - \wvec^\tr\eff(-i))\wvec \wvec^\tr - A \,.
    \end{align}
    When both producers exert equal effort, the Hessian simplifies to $f'_\noise(\wvec^\tr \vf)\wvec \wvec^\tr - A$. By assumption $\vf > 0$, $\wvec \geq 0$, and $||\wvec||=1$, which ensures that the dot product $\wvec^\tr \vf$ is positive. When $\wvec^\tr\vf > 0$, the derivative of the zero-mean Gaussian density $f'_\noise(\wvec^\tr \vf)$ is negative, making the Hessian negative-definite. Consequently, $\eff(1) = \eff(-1) = f_\wvec(\wvec^\tr \vf)A^{-1}\wvec$ represents the unique equilibrium.
\end{proof}

\subsection*{Proof of \Cref{cor:user-strategy}}
\begin{proof}
    From \Cref{prop:eq}, the equilibrium strategy for producers is symmetric: $\eff^*(1) = \eff^*(-1)$. When the strategies are symmetric, then user utility $\mathcal{U}_{\text{user}}(\eff^*(1), \eff^*(-1); \wvec)$ is equal to user utility without strategic adaptation $\mathcal{U}_{\text{user}}(\mathbf{0}, \mathbf{0}; \wvec)$. Thus, for users, the results from the non-strategic setting still hold in the strategic setting.
\end{proof}

\subsection*{Proof of \Cref{prop:prod-vector}}
\begin{proof}
    By \Cref{prop:eq}, the unique and symmetric equilibrium strategy for producers is $\eff^{*}(1) = \eff^{*}(-1) = f_\noise(\wvec^\tr \vf)A^{-1}\wvec$. Thus, producer welfare at equilibrium is equal to
\begin{align}
& \mathcal{W}_{\text{prod}}(\eff^*(1), \eff^*(-1); \wvec) \\
&  = \frac{1}{2} - c(\eff^*(1)) \\
& = \frac{1}{2} - \frac{f_\noise(\wvec^\tr \vf)^2}{2} \wvec ^\tr A^{-1} \wvec \,,
\end{align}
and the optimal weight vector minimizes $f_\noise(\wvec^\tr \vf)^2 \wvec^\tr A^{-1}\wvec$.
\end{proof}

\subsection*{Proof of \Cref{thm:prod-welf}}
\begin{proof}
By \Cref{prop:eq}, the unique and symmetric equilibrium strategy for producers is $\eff_\wvec^{*}(1) = \eff_\wvec^{*}(-1) = f_\noise(\wvec^\tr \vf)A^{-1}\wvec$ where $f_\noise$ is the density of the difference in noise terms $\noise(\wvec) = \wvec^\tr \evec(-1) - \wvec^\tr \evec(1) \sim \mathcal{N}(0, 2\wvec^\tr \wvec)$. Thus, producer welfare at equilibrium is equal to
\begin{align}
\mathcal{W}_{\text{prod}}(\eff_\wvec^*(1), \eff_\wvec^*(-1); \vf, \Sigma, A, \wvec) & = \frac{1}{2} - \frac{1}{2}c(\eff^*(1)) - \frac{1}{2}c(\eff^*(-1)) \\
& = \frac{1}{2} - \frac{ f_\noise(\wvec^\tr \vf)^2}{2} \wvec ^\tr A^{-1} \wvec \\
& = \frac{1}{2} - \frac{1}{4 \sqrt{\pi} \wvec ^\tr \Sigma \wvec} \exp \left (- \frac{(\wvec^\tr \vf)^2}{\wvec^\tr \Sigma \wvec}  \right ) \wvec ^\tr A^{-1} \wvec \,.
\end{align}
From the above expression, it is clear that for any fixed weight vector $\wvec$, producer welfare \\ $\mathcal{W}_{\text{prod}}(\eff^*(1), \eff^*(-1); \vf, \Sigma, A, \wvec)$ is monotonically increasing as the strategy-robustness $A_{jj}$ or \\ value-faithfulness $\vf_j$ of any behavior $j \in [k]$ increases. Thus, producer welfare under the optimal weight vector $\mathcal{W}_{\text{prod}}^*(\vf, \Sigma, A)$ must also be monotonically increasing in strategy-robustness and value-faithfulness. 

However, the optimal producer welfare $\mathcal{W}_{\text{prod}}^*(\vf, \Sigma, A)$ is not necessarily monotonic in the behavior's variance $\Sigma_{jj}$. The variance only affects producer welfare by changing $f_\noise(\wvec^\tr \vf)$, the density of the difference in noise terms $\noise(\wvec) \sim \mathcal{N}(0, 2\wvec^\tr \wvec)$. Let $\sigma^2 = 2\wvec^\tr\Sigma\wvec$ be the variance of $\noise(\wvec)$. For any fixed weight vector $\wvec \in \mathbb{R}_{\geq 0}^\numb$, as $\sigma^2$ approaches $0^{+}$ or $+\infty$, producer welfare approaches its maximum possible value: 
\begin{align}
& \lim_{\sigma^2 \rightarrow 0^{+}} \mathcal{W}_{\text{prod}}(\eff_\wvec^*(1), \eff_\wvec^*(-1); \vf, \Sigma, A, \wvec) = 1/2 \,, \\
& \lim_{\sigma^2 \rightarrow \infty} \mathcal{W}_{\text{prod}}(\eff_\wvec^*(1), \eff_\wvec^*(-1); \vf, \Sigma, A, \wvec) = 1/2 \,.
\end{align}
Thus, the optimal producer welfare also approaches the maximum possible value as $\sigma^2$ approaches either $0^{+}$ or $+\infty$:
\begin{align}
& 1/2 \geq \lim_{\sigma^2 \rightarrow 0^+} \mathcal{W}_{\text{prod}}^*(\vf, \Sigma, A) \geq \lim_{\sigma^2 \rightarrow 0^+} \mathcal{W}_{\text{prod}}(\eff_\wvec^*(1), \eff_\wvec^*(-1); \vf, \Sigma, A, \wvec) =  1/2\,, \\
& 1/2 \geq \lim_{\sigma^2 \rightarrow \infty} \mathcal{W}_{\text{prod}}^*(\vf, \Sigma, A) \geq  \lim_{\sigma^2 \rightarrow \infty} \mathcal{W}_{\text{prod}}(\eff_\wvec^*(1), \eff_\wvec^*(-1); \vf, \Sigma, A, \wvec) = 1/2 \,,
\end{align}
Therefore, the only way for optimal producer welfare to be monotonic in a behavior's variance is if it is constant over $\sigma^2 \in (0, \infty)$, i.e., is always equal to $1/2$, for any given value-faithfulness vector $\vf$ or cost matrix $A$. This is clearly untrue in general, and thus, optimal producer welfare is not necessarily monotonic in a behavior's variance $\Sigma_{jj}$.

Finally, for any fixed weight vector $\wvec$, we have that producer welfare approaches the maximal possible value as any of the three aspects of behavior go to $+\infty$:
 \begin{align}
 \lim_{z \rightarrow \infty} \mathcal{W}_{\text{prod}}(\eff_\wvec^*(1), \eff_\wvec^*(-1); \vf, \Sigma, A, \wvec) = 1/2
 \end{align}
for any $z \in \{A_{jj}, \vf_j, \Sigma_{jj} \mid j \in [\numb] \}$. Furthermore,
 \begin{align}
 1/2 \geq \lim_{z \rightarrow \infty} \mathcal{W}_{\text{prod}}^*(\vf, \Sigma, A) \geq \lim_{z \rightarrow \infty} \mathcal{W}_{\text{prod}}(\eff_\wvec^*(1), \eff_\wvec^*(-1); \vf, \Sigma, A, \wvec) = 1/2 \,,
 \end{align}
 and thus, the optimal producer welfare also approaches the maximum possible value: $\lim_{z \rightarrow \infty} \mathcal{W}_{\text{prod}}^*(\vf, \Sigma, A) = 1/2$.
\end{proof}
\section{Synthetic experiments} \label{app:sim}
\subsection*{Parameters for \Cref{fig:vf-var-tradeoff}}
The parameters used for the simulation are
\begin{align}
    & \text{value: } v(1) = 1, v(-1) = 0\,, \\
    & \text{behavioral biases: } \bias(-1) = \mathbf{0}, \bias(1)_1 = 0.75\,, \\
    & \text{variance: } \Sigma_{11} = 1\,, \\
    & \text{cost of manipulation: } A=I \,.
\end{align}
 The bias $\bias(1)_2$ and variance $\Sigma_{22}$ of the second behavior is adjusted so that the second behavior has the relative value-faithfulness and variance given by the $x$ and $y$-axes. 

\subsection*{Parameters for \Cref{fig:main}}
The default parameters used for each of the subplots are
\begin{align}
    & \text{value: } v(1) = 1, v(-1) = 0\,, \\
    & \text{behavioral biases: } \bias(-1) = \mathbf{0}, \bias(1)_1 = \bias(1)_2 = 0.75\,, \\
    & \text{variance: } \Sigma_{11} = 1, \Sigma_{22} = 3\,, \\
    & \text{cost of manipulation: } A=I \,.
\end{align}
The three subplots are generated by adjusting the bias $\bias(1)_2$, variance $\Sigma_{22}$, or cost $A_{22}$ of the second behavior so that it has the relative value-faithfulness, variance, or strategy-robustness given by the $x$-axis.


\begin{figure}[t]
  \centering
  \begin{subfigure}[b]{0.48\textwidth}
    \includegraphics[width=\textwidth]{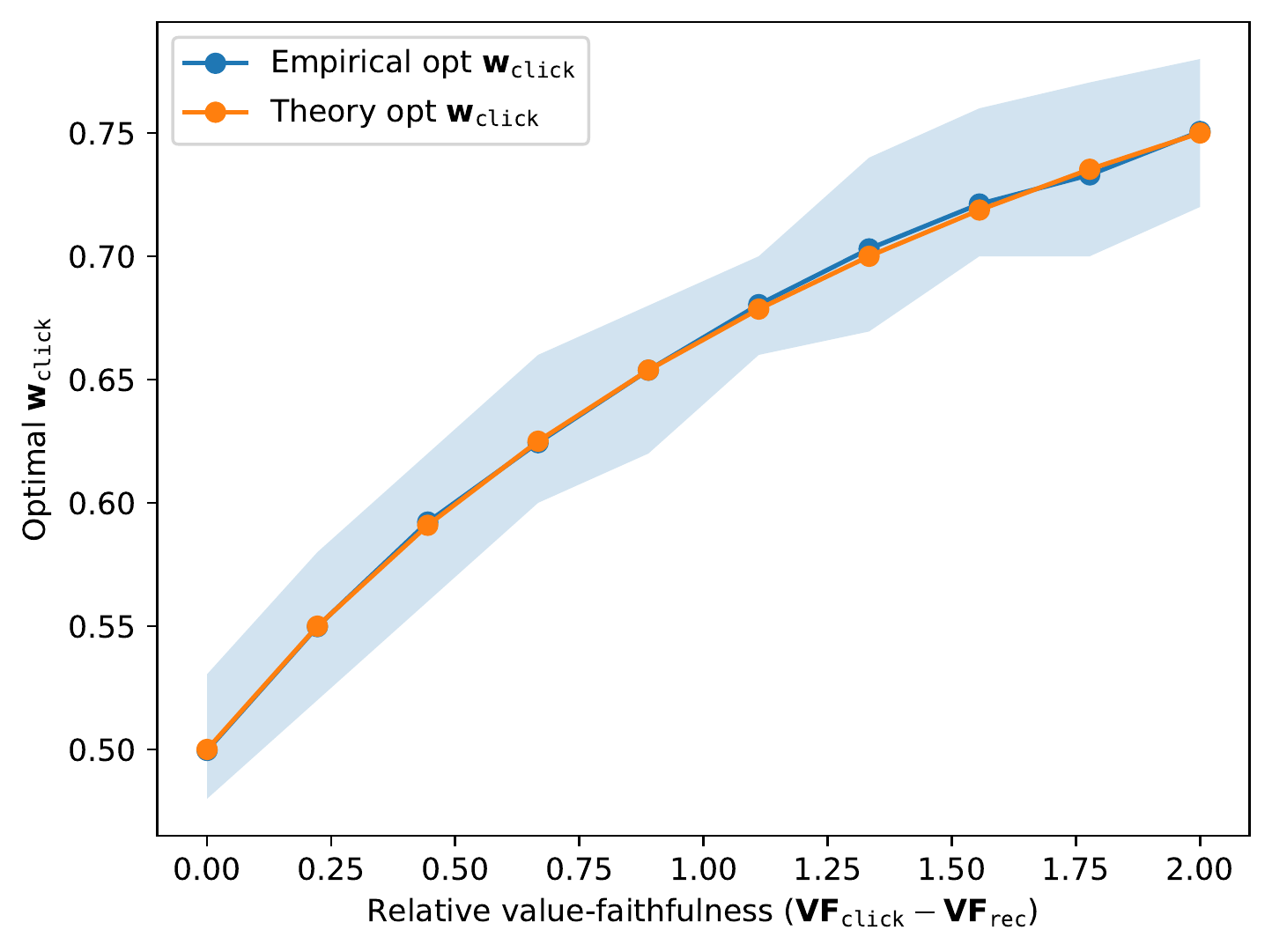}
  \end{subfigure}
  \hfill 
  \begin{subfigure}[b]{0.48\textwidth}
    \includegraphics[width=\textwidth]{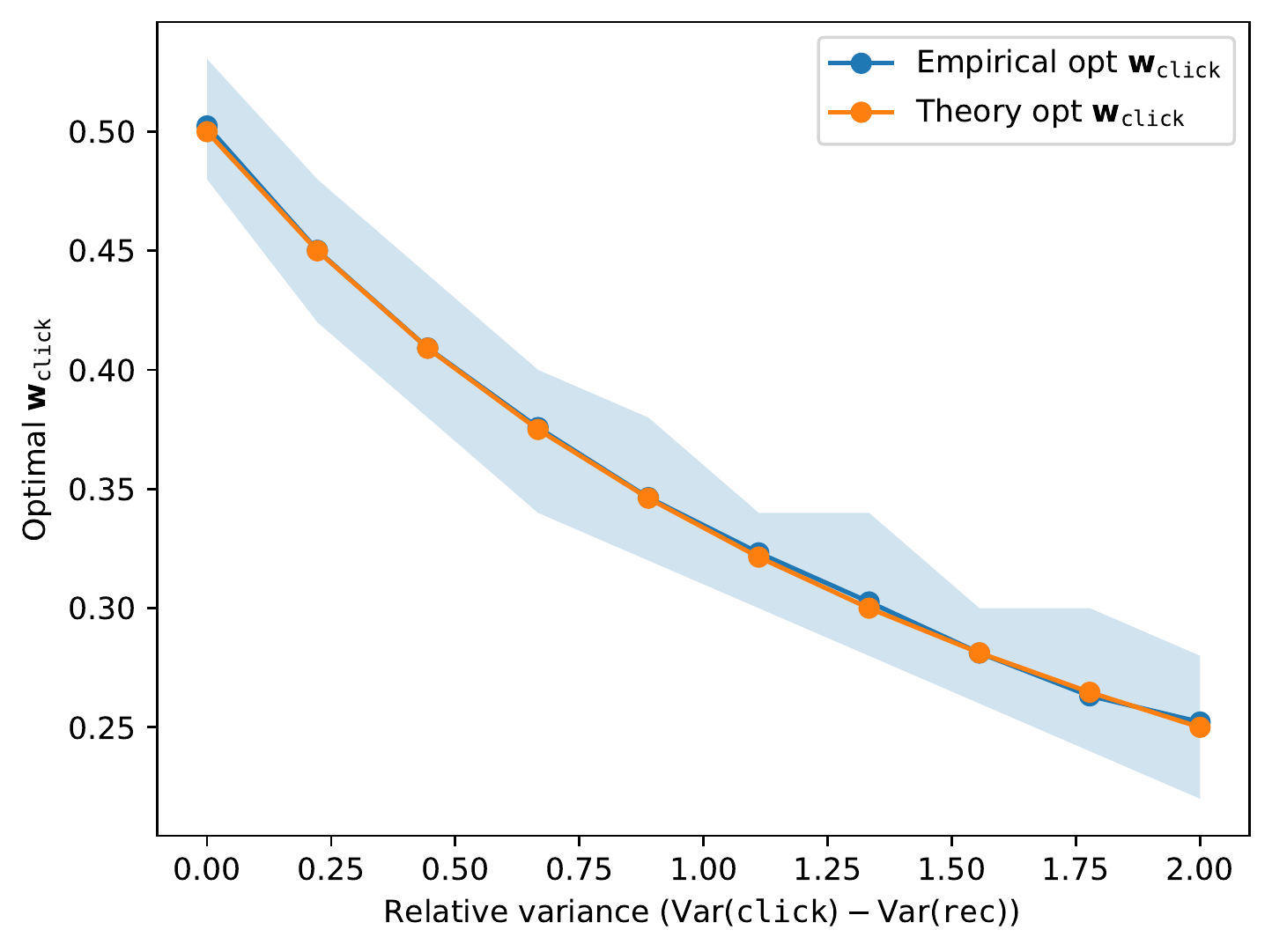}
  \end{subfigure}
  \caption{The optimal weight vector as a function of value-faithfulness and variance in the homogeneous setting. The default parameters for the simulations are $\mu_{\texttt{click}} = \mu_{\texttt{rec}} = 1$ and $\Sigma_{11} = \Sigma_{22} = 2$. The left figure is generated by plotting the optimal weight vector as $\mu_{\texttt{rec}}$ increases (and consequently, when value-faithfulness increases), and the right figure is generated by increasing $\Sigma_{22}$. The confidence bands show 95\% confidence intervals based on 100 simulations of the data.}
  \label{fig:hom-sim}
\end{figure}

\begin{figure}[t]
  \centering
  \begin{subfigure}[b]{0.48\textwidth}
    \includegraphics[width=\textwidth]{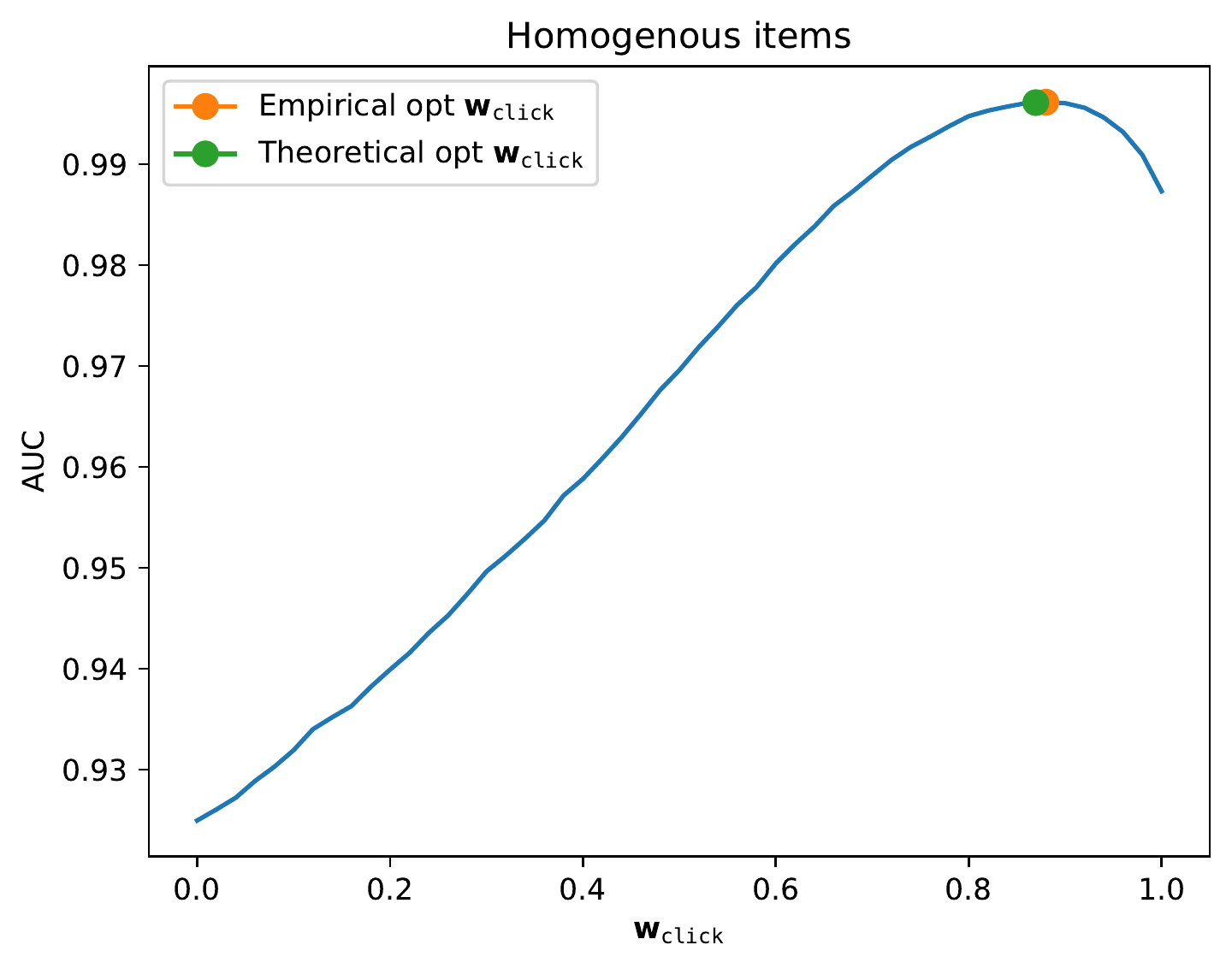}
  \end{subfigure}
  \hfill 
  \begin{subfigure}[b]{0.48\textwidth}
    \includegraphics[width=\textwidth]{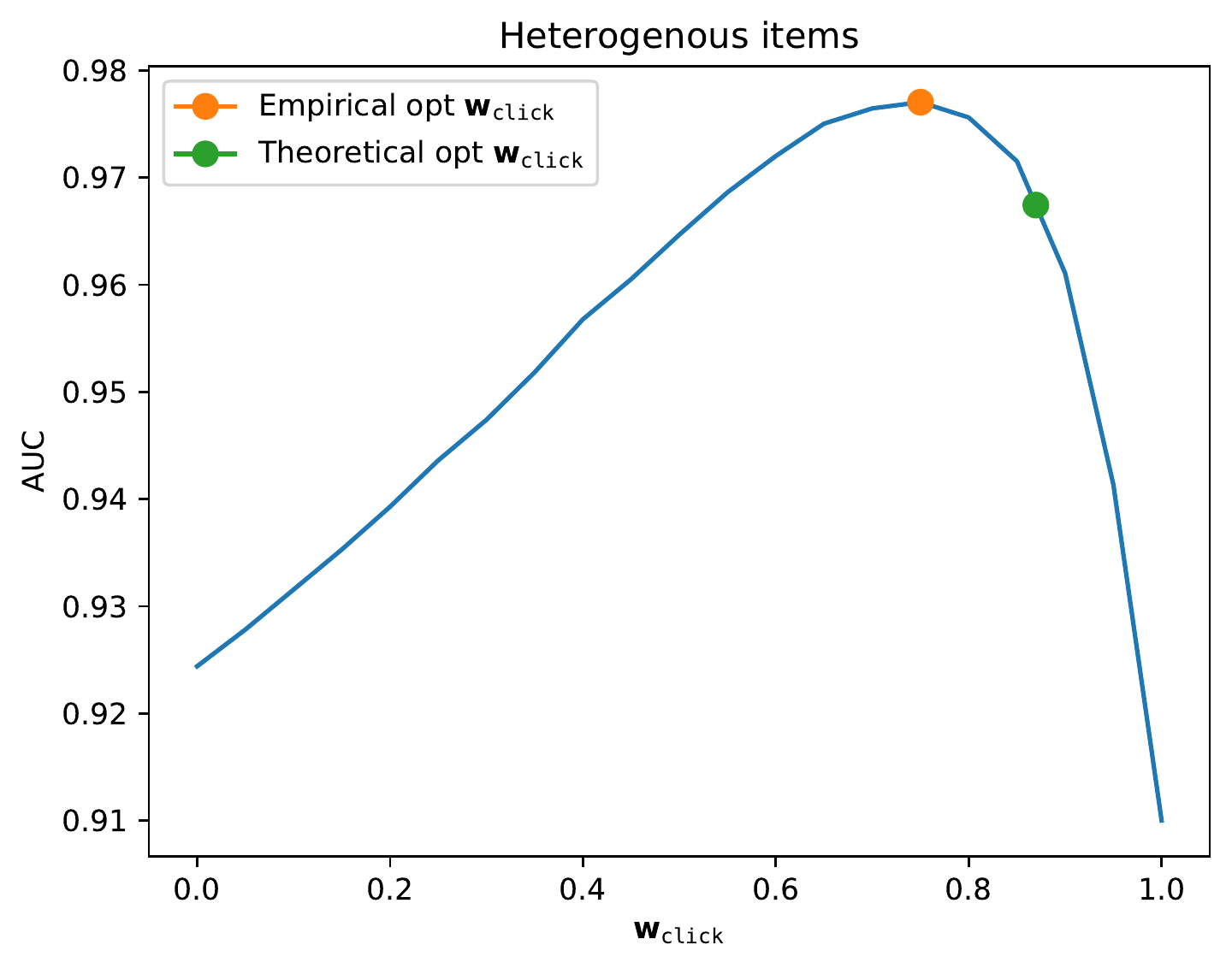}
  \end{subfigure}
  \caption{A comparison of the empirical and theoretical optimal weight vector in the homogeneous and heterogeneous setting. In both settings, the variance on click is $\Sigma_{11} = 0.1$ and the variance on recommend is $\Sigma_{22} = 2$. In the homogeneous setting, $\mu_{\texttt{click}} = 1$ and $\mu_{\texttt{rec}} = 3$. In the heterogeneous setting, $\alpha_{\texttt{click}} = 0$, $\beta_{\texttt{click}} = 2$, $\alpha_{\texttt{rec}} = 2$, and $\beta_{\texttt{click}} = 4$. Thus, the mean predictions across items are the same in both the homogeneous and heterogeneous setting: $\mu_{\texttt{click}} = (\mu_{\texttt{i, click}})/n_{+}$ and $\mu_{\texttt{rec}} = (\mu_{\texttt{i, rec}})/n_{+}$. However, while the theory optimal weight vector and empirical optimal weight vector closely match in the homogeneous setting, they have a distinct gap in the heterogeneous setting.}
  \label{fig:het-sim}
\end{figure}

\subsection{Additional simulations with $n > 2$ items}
In simulation, we consider the non-strategic setting, i.e., when $\eff \triangleq \mathbf{0}$, and extend it to a setting with $n$ producers or items. Here, the user values $n_{+}$ items and doesn't value $n_{-}$ items. All valued items are assumed to have the same positive value $v_{+}$ while unvalued items have a value of zero. In the two-item setting, we defined a user's utility as the probability that the higher-valued item is ranked first (\Cref{eq:user-utility}). Then, a natural metric to optimize for in the $n$ item setting is the probability that a randomly-picked valued item is ranked above a randomly-picked unvalued item, i.e., the AUC. 

The user can interact with the items using two different behaviors: (1) \emph{click} and (2) \emph{recommend}. In our simulations, we assume that recommend is more value-faithful than click but also higher variance. To extend value-faithfulness to the setting with $n$ items, we define the mean value-faithfulness $\overline{\vf}$ as the mean difference in behavior scores between valued items and unvalued items:
\begin{align}
    & \overline{\vf} = \frac{1}{n_{+}}\sum_{i : v(i) > 0} \yvec(i) - \frac{1}{n_{-}}\sum_{i : v(i) = 0} \yvec(i)\,.
\end{align}
We investigate two settings: one in which, given their value, each item has the same mean behavior predictions, and the other, in which items are heterogeneous. In both, we compare (a) the weight vector that maximizes the empirical AUC and (b) the user-optimal weight vector given by Theorem \ref{thm:user-w} (in which we substitute $\overline{\vf}$ for $\vf$). 

In the homogeneous setting, the predictions are simulated as
\begin{align}
     \yvec(i) \sim \begin{cases} \mathcal{N}(\mathbf{0}, \Sigma) & v(i) = 0 \\
      \mathcal{N}\left (\begin{bmatrix} \mu_{\texttt{click}} \\ \mu_{\texttt{rec}} \end{bmatrix}, \Sigma \right) & v(i) = v_{+}
    \end{cases} \,,
\end{align}
i.e., all unvalued items have the same mean prediction of $0$ for both clicks and recommend while all valued items have the same mean prediction $\mu_{\texttt{click}} > 0$ and $\mu_{\texttt{rec}} > 0$ for clicks and recommend.

In the heterogeneous setting, the behavior predictions for unvalued items are simulated the same way, i.e. $\yvec(i) \sim (\mathbf{0}, \Sigma)$ for $i$ such that $v(i) = 0$. But for valued items, the mean of the behavior predictions is heterogeneous, i.e.,
\begin{align}
    \mu_{i, \texttt{click}} & \sim \mathrm{Unif}[\alpha_{\texttt{click}}, \beta_{\texttt{click}}]\,, \\
    \mu_{i, \texttt{rec}} & \sim \mathrm{Unif}[\alpha_{\texttt{rec}}, \beta_{\texttt{rec}}]\,, \\
     \yvec(i) & \sim \mathcal{N}\left (\begin{bmatrix} \mu_{i, \texttt{click}} \\ \mu_{i, \texttt{rec}} \end{bmatrix}, \Sigma \right)
   \,.
\end{align}

Figure \ref{fig:hom-sim} shows that in the homogeneous setting, the theory-optimal weight vector and the empirical-optimal weight vector match closely even as value-faithfulness and variance change. However, Figure \ref{fig:het-sim} demonstrates that in the heterogeneous setting, the theory-optimal weight vector and the empirical-optimal weight vector may not necessarily coincide.

\section{Recommending URLs on Facebook} \label{app:deming}
In this section, we explain how we create an estimator of the probability that users react to a URL $u$ with a given behavior $i$.

Let $y_{uig}$ be the number of times that group $g$ reacted to URL $u$ with behavior $i$, and let $x_{ug}$ be the number of times that group $g$ viewed URL $u$. By summing over different demographic groups, for any URL $u$, we can calculate $y_{ui} \coloneqq \sum_{g} y_{uig}$, the total number of times that users engaged with the URL using behavior $i$. Similarly, we can determine the total number of times that users viewed the URL, $x_u = \sum_{g} x_{ug}$. Under normal circumstances, a straight-forward estimator for the probability of a user engaging with URL $u$ using behavior $i$ might be the ratio $y_{ui} / x_{u}$. However, because of the differential privacy noise introduced into the dataset, it is possible for the engagement counts $y_{ui}$ to be greater than the total number of views $x_{u}$, or even for either $y_{ui}$ or $x_{u}$ to be negative. 

To accurately estimate engagement probabilities while taking into account the differential privacy noise, we leverage the statistical literature on errors-in-variables models. Specifically, we use Deming regression~\citep{adcock1878,kummell1879,Deming_1943}, a regression method that adjust for known Gaussian measurement errors in both the independent and dependent variables. This approach is particularly suitable for our analysis as Gaussian noise was deliberately added to the data to ensure differential privacy, and its parameters are known.

Specifically, we fit the following Deming regression model:
\begin{align}
    y_{uig} = \beta_{ui} x_{ug}\,, \\
    x_{ug} = \bar{x}_{ug} + \epsilon^{x}_{ug}\,, \\
    y_{uig} = \bar{y}_{uig} + \epsilon^{y}_{uig}\,,
\end{align}
where $\beta_{ui}$ is the coefficient being estimated, the variables $\bar{y}_{uig}$ and $\bar{x}_{ug}$ the true, unobserved engagement and view counts (before Gaussian noise was added), and the variables $\epsilon^{x}_{ug} \sim \mathcal{N}(0, \sigma_x^2)$ and $\epsilon^{y}_{uig} \sim \mathcal{N}(0, \sigma_i^2)$ are the independent noise terms. In the Facebook URLs dataset, the standard deviation of the noise added to `views' is 2228, to `likes' is 22, and to the five emoji reactions is 10~\citep{fburls}.

In our model, the coefficient $\beta_{ui}$ is constrained to be in the range $[0, 1]$ and serves as our estimate of the probability that a user engages with URL $u$ using behavior $i$.
\section{Measuring the Three Aspects} \label{app:measuring-aspects}

Here, we outline ways that the platform could practically measure each of the three aspects we study: value-faithfulness, variance, and strategy-robustness. After quantitatively or qualitatively measuring the three aspects, it may become clear that certain sets of weights are more relevant to test than others. For example, our theoretical results (\Cref{cor:user-strategy}) suggest that if behavior $i$ is both less value-faithful and higher variance than behavior $j$, then for the user, it is better for behavior $j$ to have higher weight than behavior $i$. While our model is stylized, such insights may be useful heuristics in practice.

\paragraph{Value-faithfulness.} The value-faithfulness of each behavior could be measured through user surveys in which users are explicitly asked whether or not they value a piece of content~\citep{andre2012gives,milli2023twitter}. Many platforms have used such kinds of surveys, e.g., Youtube measures what they call \emph{valued watchtime} through user surveys ~\citep{goodrow_2021}. Facebook explicitly used surveys measuring how much users value different kinds of interactions in choosing the weights: \emph{``the base weight of all the interactions are derived based on producer-side experiments which measure value to the originator/producer (of the content)''} \citep{cameron_wodinsky_degeurin_germain_2022,litt2020meaningful}.\footnote{The quote is from a leaked document from the Facebook Files titled ``The Meaningful Social Interactions Metric Revisited: Part Two''.}

Suppose that a platform has a dataset $\mathcal{D}$ consisting of, for each surveyed item, the behaviors the user engaged in on that item $\ovec \in \{0, 1\}^\numb$, and their associated value rating $v \in \{0, 1\}$. Then, the value-faithfulness of behavior $j$ could be estimated as the empirical probability of observing behavior $j$ given that the item was valued minus the empirical probability of observing behavior $j$ given that the item was not valued:
\begin{align}
    \widehat{\vf}_{j} = & \frac{1}{n}\left \lvert \{ (\ovec, v) \in \mathcal{D} : \ovec_j = 1, v = 1 \} \right \rvert \\
    & - \frac{1}{n}\left \lvert \{ (\ovec, v) \in \mathcal{D} : \ovec_j = 1, v = 0 \} \right \rvert \,.
\end{align}
This is a simple estimate that aggregates all users together in determining the value-faithfulness of a behavior. One could imagine more sophisticated approaches that determine the value-faithfulness of each behavior in a way that is personalized to each user~\citep{maghakian2022personalized}.

\paragraph{Variance.} The average variance of the behavior predictors can be measured empirically through standard bootstrap resampling~\citep{efron1992bootstrap}. Assume that the platform learns a predictor $f : \mathcal{X} \rightarrow [0, 1]^\numb$ that maps user and item features to predictions of whether the user will engage with the item through each behavior. The platform learns the predictor from a dataset of $n$ historical user-item interactions $\mathcal{H} = \{(\xvec, \ovec)\}$ where each user-item interaction is represented by user-item features $\xvec \in \mathcal{X}$ and observed behaviors $\ovec \in \{0, 1\}^\numb$. Then, the variance of predicting behavior $j$ on datapoint $\xvec$ is $\mathrm{Var}_j(\xvec) = \ex_{\mathcal{H}}[(f_j(\xvec) - \ex_{\mathcal{H}}[f_j(\xvec)])^2]$ where the expectation is taken over different bootstrap samples of $\mathcal{H}$. The goal is to estimate the mean variance $V(j) = \ex_{\xvec}[\mathrm{Var}_j(x)]$ across datapoints.

Let $f^i, \dots, f^m$ be predictors learned from different bootstrap samples of the dataset $\mathcal{H}$ and $\bar{f} \triangleq (1/m) \sum_{i=1}^m f^i$. Then, an estimate for the mean variance $V(j)$ of the $j$-th behavior predictor is
\begin{align}
\widehat{V}(j) = \frac{1}{mn} \sum_{(\xvec, \ovec) \in \mathcal{H}} \sum_{i=1}^{m}  (f^i(\xvec) - \bar{f}(\xvec))^2 \,.
\end{align}

\paragraph{Strategy-robustness.} Strategy-robustness is, perhaps, the most difficult to estimate directly as it requires anticipating how producers will strategically adapt to changes in the objective function. This strategic adaptation typically takes time and evolves as producers share strategies. Though it is difficult to quantitatively measure strategy-robustness, it may be evaluated more qualitatively through domain knowledge and producer testimony. For example, after Youtube switched to focusing on optimizing watchtime instead of views, they wrote a blog post stating that ``many of the tactics we’ve heard from creators to optimize for YouTube’s discovery features may in fact backfire'' such as making videos shorter~\citep{meyerson2012youtube}. Testimony and interviews from producers can help designers understand whether a given behavior is likely to be successfully gamed or not.
\end{appendix}

\end{document}